%% file: main.tex
\documentclass{article}

    \PassOptionsToPackage{numbers, compress}{natbib}
\usepackage[final]{neurips_2023}
\usepackage{algorithm}
\usepackage{algorithmic}

\usepackage[utf8]{inputenc} % allow utf-8 input
\usepackage[T1]{fontenc}    % use 8-bit T1 fonts
\usepackage{hyperref}       % hyperlinks
\usepackage{url}            % simple URL typesetting
\usepackage{booktabs}       % professional-quality tables
\usepackage{amsfonts}       % blackboard math symbols
\usepackage{nicefrac}       % compact symbols for 1/2, etc.
\usepackage{microtype}      % microtypography
\usepackage{xcolor}         % colors

\usepackage{graphicx}
\usepackage{wrapfig}
\usepackage{caption}
\usepackage{subcaption}
\usepackage{multirow}
\usepackage{threeparttable}

\usepackage{amsmath}
\usepackage{amssymb}
\usepackage{mathtools}
\usepackage{amsthm}
\usepackage{enumitem}

\usepackage{musicography}

\theoremstyle{plain}
\newtheorem{theorem}{Theorem}[section]

\newtheorem{lemma}[theorem]{Lemma}

\theoremstyle{definition}

\newtheorem{assumption}[theorem]{Assumption}
\theoremstyle{remark}

\input{math_commands.tex}

\newcommand{\name}{{SIPO}}
\newcommand{\namerbf}{{SIPO-RBF}}
\newcommand{\namewd}{{SIPO-WD}}

\newcommand{\cmrd}[1]{{#1}}

\newcommand{\sctwo}{{StarCraft \uppercase\expandafter{\romannumeral2}}}

\title{Iteratively Learn Diverse Strategies with State Distance Information}

\author{%
  Wei Fu$^{1\textrm{\fl}}$, Weihua Du\thanks{Equal Contribution}\,\,$^{1}$, Jingwei Li$^{* 1}$, Sunli Chen$^{1}$, Jingzhao Zhang$^{12}$, Yi Wu$^{12\textrm{\sh}}$ \\
  $^{1}$ IIIS, Tsinghua University, $^{2}$ Shanghai Qi Zhi Institute\\
  $^{\textrm{\fl}}$ \texttt{fuwth17@gmail.com}, $^{\textrm{\sh}}$\texttt{jxwuyi@gmail.com}\\
}

\begin{document}

\maketitle

\begin{abstract}
In complex reinforcement learning (RL) problems, policies with similar rewards may have substantially different behaviors. It remains a fundamental challenge to optimize rewards while also discovering as many \emph{diverse} strategies as possible, which can be crucial in many practical applications. Our study examines two design choices for tackling this challenge, i.e., \emph{diversity measure} and \emph{computation framework}. First, we find that with existing diversity measures, visually indistinguishable policies can still yield high diversity scores. To accurately capture the behavioral difference, we propose to incorporate the state-space distance information into the diversity measure. In addition, we examine two common computation frameworks for this problem, i.e., population-based training (PBT) and iterative learning (ITR). We show that although PBT is the precise problem formulation, ITR can achieve comparable diversity scores with higher computation efficiency, leading to improved solution quality in practice. Based on our analysis, we further combine ITR with two tractable realizations of the state-distance-based diversity measures and develop a novel diversity-driven RL algorithm, \emph{State-based Intrinsic-reward Policy Optimization} (SIPO), with provable convergence properties. We empirically examine SIPO across three domains from robot locomotion to multi-agent games. In all of our testing environments, SIPO consistently produces strategically diverse and human-interpretable policies that cannot be discovered by existing baselines.
\end{abstract}

\input{1-intro.tex}

\input{2-related.tex}

\input{3-prelim.tex}

\input{4-analysis-component.tex}

\input{5-method.tex}

\input{6-experiments.tex}

\vspace{-1mm}
\section{Conclusion}
\label{sec:conclusion}
\vspace{-1mm}
We tackle the problem of discovering diverse high-reward policies in RL.
First, we demonstrate concrete failure cases of existing diversity measures and propose a novel measure that explicitly compares the distance in state space.
Next, we present a thorough comparison between PBT and ITR and show that {ITR} is much easier to optimize and can derive solutions with comparable quality to PBT.
% Moreover, we also demonstrate concrete failure cases for action-based diversity measures. 
Motivated by these insights, we combine {ITR} with a state-distance-based diversity measure to develop {\name}, which has provable convergence and can efficiently discover a wide spectrum of human-interpretable strategies in a wide range of environments.

\textbf{Limitations:}
First, we assume direct access to an object-centric state representation.
When such a representation is not available (e.g., image-based observations), representation learning becomes necessary and algorithm performance can be affected by the quality of the learned representations.
% {We leave further studies on representation learning as our future direction.}
Second, because {ITR} requires sequential training, {the wall clock time of {\name} can be longer than the PBT alternatives when fixing the total number of training samples.}
The acceleration of ITR remains an open challenge.

\textbf{Future Directions:}
Besides addressing the above limitations, we suggest three additional future directions based on our paper.
First, a consensus on the best algorithmic formulation of distinct solutions in RL remains elusive. It is imperative to understand diversity in a more theoretical manner.
Second, while this paper focuses on single-agent and cooperative multi-agent domains, extending SIPO to multi-agent competitive games holds great potential.
Finally, although SIPO/ITR enables open-ended training, it is worth studying how to determine the optimal population size to better balance resources and the diversity of the resulting population.

\section*{Acknowledgement}

This project is partially supported by 2030 Innovation Megaprojects of China (Programme on New Generation Artificial Intelligence) Grant No. 2021AAA0150000.

%%%%%%%%%%%%%%%%%%%%%%%%%%%%%%%%%%%%%%%%%%%%%%%%%%%%%%%%%%%%%%%
\bibliography{main}
\bibliographystyle{plainnat}
%%%%%%%%%%%%%%%%%%%%%%%%%%%%%%%%%%%%%%%%%%%%%%%%%%%%%%%%%%%%%%%

%%%%%%%%%%%%%%%%%%%%%%%%%%% appendix begin %%%%%%%%%%%%%%%%%%%%%%%%%%%%%%%%%%%%
\newpage
\appendix
\section{Project Website}
\label{app:web}
Check \url{https://sites.google.com/view/diversity-sipo} for GIF demonstrations.

\input{A2-additional.tex}

\input{A3-env.tex}

\input{A4-imple.tex}

\section{Proofs}

\subsection{Proof of theorem \ref{thm:2delta}}
\label{app:proof1}

\textbf{Theorem~\ref{thm:2delta}. }
\textit{
Assume $D$ is a distance metric. Denote the optimal value of Problem~\ref{eq:pbt-co} as $T_1$.
Let $T_2=\sum_{i=1}^M J(\tilde{\pi}_i)$ where
\begin{align}
    \begin{split}
        \tilde{\pi}_i &=\arg\max_{\pi_i} \quad  J(\pi_i)\\
        \textrm{s.t.} & \quad  D(\pi_i,\tilde{\pi}_j)\ge\delta/2,\quad\,\forall 1\le j< i
        % \label{eq:iter-thm1}
    \end{split}\tag{3}
\end{align}
for $i=1,\dots,M$, then $T_2\ge T_1$.
}

\begin{proof}
Suppose the optimal solution of Problem~\ref{eq:pbt-co} is $\pi_{1}, \pi_{2}, ... , \pi_{M}$ satisfying $J(\pi_1) \ge J(\pi_2) \ge ... \ge J(\pi_M)$ and the optimal solution of Problem~\ref{eq:iter-thm1} is $\Tilde{\pi}_{1}, \Tilde{\pi}_{2}, ... , \Tilde{\pi}_{M}$ satisfying $J(\Tilde{\pi}_1) \ge J(\Tilde{\pi}_2) \ge ... \ge J(\Tilde{\pi}_M)$.

Assume the contrary that Thm.~\ref{thm:2delta} is not true, which means $\sum_{i=1}^M J(\pi_i) = T_1 > T_2 = \sum_{i=1}^M J(\Tilde{\pi}_i)$. Then we choose the smallest number $N\le M$ that satisfies 
$$\sum_{i=1}^N J(\pi_i) > \sum_{i=1}^N J(\Tilde{\pi}_i).$$

By $T_1 > T_2$ we know that $N$ exists. In addition, because Problem~\ref{eq:iter-thm1} solves unconstrained RL in the first iteration, we know that $\Tilde{\pi}_1 = \argmax_\pi J(\pi)$ and then $J(\pi_1) \le J(\Tilde{\pi}_1)$. Therefore, $N \ge 2$. 

Suppose $J(\pi_N) \le J(\Tilde{\pi}_N)$. Then we have 

$$\sum_{i=1}^{N-1} J(\pi_i) > \sum_{i=1}^{N-1} J(\Tilde{\pi}_i).$$

Contradicting the fact that $N$ is the smallest number satisfies that equation.

Hence, we know that $J(\pi_N) > J(\Tilde{\pi}_N)$. Then 
$$J(\pi_1) \ge J(\pi_2) \ge ... \ge J(\pi_N) > J(\Tilde{\pi}_N). $$

Consider the optimization problem of $\Tilde{\pi}_N$:

\begin{align*}
    \begin{split}
        \tilde{\pi}_N &=\arg\max_{\pi} \quad  J(\pi)\\
        \textrm{s.t.} & \quad  D(\pi,\tilde{\pi}_j)\ge\delta/2,\quad\,\forall 1\le j< N.
        \label{eq:iter-thm1}
    \end{split}
\end{align*}

This optimization does not find $\{\pi_1,\dots,\pi_N\}$ but find $\Tilde{\pi}_N$, which means that for each $\pi_i$, $1 \le i \le N$, there exists $1 \le j_i < N$ such that $D(\pi_i,\tilde{\pi}_{j_i}) < \delta/2$.
Otherwise, we will get the solution of the above problem as $\pi_i$ instead of $\Tilde{\pi}_N$.

By the Pigeonhole Principle, we know that there exist two indexes $i_1\in[N]$ and $i_2\in[N]$ $(i_1\neq i_2)$ such that $j_{i_1}=j_{i_2}=\hat{j}$. Then we have
$$D(\pi_{i_1},\pi_{i_2}) \le D(\pi_{i_1},\tilde{\pi}_{\hat{j}}) + D(\pi_{i_2},\tilde{\pi}_{\hat{j}}) < \delta/2 + \delta/2 = \delta ,$$
where the inequality follows by the triangle inequality of the distance function. 

It contradict with the fact that $D(\pi_{i_1},\pi_{i_2}) \ge \delta$ in Problem~\ref{eq:pbt-co}.

Therefore, we prove the theorem $\sum_{i=1}^M J(\pi_i) = T_1 \le  T_2 = \sum_{i=1}^M J(\Tilde{\pi}_i)$.
\end{proof}

\subsection{Proof of Theorem~\ref{thm:converge}}
\label{app:proof2}
In this section, we consider the $i$-th iteration of {\name} illustrated in Eq.~(\ref{eq:iter-co}).
For the sake of simplicity, we use 
$a\le\bm{\lambda}\le b$ for vector $\bm{\lambda}$ to denote each component of $\bm{\lambda}$ satisfies $a\le\lambda_i\le b$, where $a,b\in\R$.  We use $\pi$ to denote the policy we are optimizing, and $\pi_j$ $(1\le j< i)$ to denote a previously obtained policy. 
We denote the Lagrange function as $L(\pi, \bm{\lambda}) = -J(\pi) - \sum_{j=1}^{i-1}  \lambda_j\left(D(\pi,\pi_j) - \delta\right)$.

To prove Theorem~\ref{thm:converge}, we consider the following two optimization problems:
\begin{equation}
    (\pi_i,\bm{\lambda}^\star)=\arg\min_{\pi} \max_{\bm{\lambda}\ge 0} L(\pi, \bm{\lambda})
    \label{eq:lagrange}
\end{equation}
and
\begin{equation}
    (\tilde{\pi}_i,\tilde{\bm{\lambda}}^\star)=\arg\min_{\pi} \max_{0\le\bm{\lambda}\le\Lambda} L(\pi, \bm{\lambda}),
    \label{eq:lagrange-bounded}
\end{equation}
where $\Lambda=\frac{1}{\epsilon_0}$ and $\epsilon_0>0$ is sufficiently small.

%%%%%%%%%%%%%%%%%%%%%%%%%%%%%%%%%%%%%%%
%%% assumptions %%%
We make the following assumptions to prove this theorem:
\begin{assumption}
\label{assump1}
    $0\le J(\cdot)\le 1$.
\end{assumption}

\begin{assumption}
\label{assump2}
    $\forall\bm{\lambda}\ge0,\, L(\cdot,\bm{\lambda})$ is $l$-smooth and $\zeta$-Lipschitz.
\end{assumption}

%%%%%%%%%%%%%%%%%%%%%%%%%%%%%%%%%%%%%%%%%
%%% lemma
We may notice that solving the optimization problem (11) is hard because its domain is unbounded. Therefore, we make some approximations and consider the bounded optimization problem (12). First, we prove the following lemma about the value function $J$:
\begin{lemma}
\label{lem1}
$J(\pi_i) \le J(\Tilde{\pi}_i)$.
\end{lemma}
\begin{proof}
As the domain of $\bm{\lambda}$ in Eq.~\ref{eq:lagrange-bounded} is smaller than Eq.~(\ref{eq:lagrange}), 
we have $L(\pi_i, \bm{\lambda}) \ge L(\Tilde{\pi}_i, \Tilde{\bm{\lambda}})$.

By the fundamental property of Lagrange duality, we know that $L$ achieves its optimal value when $\bm{\lambda} = 0$ and the optimal value is $-J(\pi_i)$.

By the optimality of $(\Tilde{\pi}_i, \Tilde{\bm{\lambda}}^\star)$, we know that 
\begin{equation}
\label{eq:bounded-constraint-violation}
    - \sum_{j=1}^{i-1}  \Tilde{\lambda}^\star_j(D(
\Tilde{\pi}_i,\pi_j) - \delta) \ge 0.
\end{equation}

Then we have 
$$-J(\pi_i) = L(\pi_i, \bm{\lambda}^\star) \ge \Tilde{L}(\Tilde{\pi}_i, \Tilde{\bm{\lambda}}^\star) = - J(\Tilde{\pi}_i) - \sum_{j=1}^{i-1}  \Tilde{\lambda}^{\star}_j(D(
\Tilde{\pi}_i,\pi_j) - \delta) \ge - J(\Tilde{\pi}_i).$$
\end{proof}

Then we prove the distance between optimal policy $\Tilde{\pi}_i$ in problem (12) and optimal policy $\pi_i$ in problem (11) is very small:
\begin{lemma}
\label{lem2}
Under Assumption~\ref{assump1}, $D(\Tilde{\pi}_i, \pi_j) \ge \delta - \epsilon_0 ,\,\forall 1\le j< i$.
\end{lemma}

\begin{proof}
We prove this by contradiction.

Suppose there exists $1\le j_0< i $, $D(\Tilde{\pi}_i, \pi_{j_0})  <  \delta - \epsilon_0$. Then we choose $\hat{\bm{\lambda}}$ such that
\begin{equation*}
\hat{\lambda}_j=\left\{
\begin{aligned}
\Lambda& \quad  \text{$j = j_0$} \,,\\
0&  \quad \text{$1 \le j < i, j \neq j_0$} \,.\\
\end{aligned}
\right.
\end{equation*}

By the Assumption~\ref{assump1}, Eq.~(\ref{eq:bounded-constraint-violation}), and $\Lambda = \frac{1}{\epsilon_0}$, we have 

$$0 \ge - J(\pi_i) = L(\pi_i, \bm{\lambda}^\star)  \ge  L(\Tilde{\pi}_i, \Tilde{\bm{\lambda}}^\star) \ge  L(\Tilde{\pi}_i, \hat{\bm{\lambda}}) \ge -1 - \Lambda(D(\Tilde{\pi}_i, \pi_{j_0}) - \delta) > 0.$$

That is a contradiction. So we have proved that 
$$D(\Tilde{\pi}_i, \pi_j) \ge \delta - \epsilon_0 ,\quad\,\forall 1\le j< i .$$
\end{proof}

From the deduction above, we get the following approximation lemma:
\begin{lemma}
\label{lem3}
Denote the optimal solution of Eq.~\ref{eq:lagrange} and Eq.~\ref{eq:lagrange-bounded} as $(\pi_i, \lambda)$ and $ (\Tilde{\pi}_i, \Tilde{\lambda})$ respectively. Then we have the following approximation about the optimal value and distance:
\begin{align*}
    J(\pi_i) &\le J(\Tilde{\pi}_i) \\
    D(\Tilde{\pi}_i, \pi_j) &\ge \delta - \epsilon_0 ,\quad\,\forall 1\le j< i \,
\end{align*}
\end{lemma}

\begin{proof}
    This lemma follows directly by Lemma \ref{lem1} and Lemma \ref{lem2}.
\end{proof}

Therefore, it is reasonable to consider the constrained optimization problem (12) instead of primal problem (11) because we have proved that the optimal value doesn't get smaller and the distance of policy is $\epsilon_0$-approximation of the primal problem. Finally we use the conclusion in the paper \cite{lin2020gradient} to analysis the convergence of problem (12): 
\begin{lemma}
\label{lem4}
(\cite{lin2020gradient}, Theorem 4.8)
Under Assumption~\ref{assump2}, solving Eq.~(\ref{eq:lagrange-bounded}) via two-timescale GDA with learning rate $\eta_\pi = \Theta(\epsilon^4/l^3 \zeta^2 \Lambda^2)$ and $\eta_{\bm{\lambda}} = \Theta(1/l)$ requires 
$$\mathcal{O}\left(\frac{l^3 \zeta^2 \Lambda^2 C_1}{\epsilon^6} + \frac{l^3 \Lambda^2 C_2}{\epsilon^4}\right)$$
iterations to converge to an $\epsilon$-stationary point $\pi_i^\star$, where $C_1$ and $C_2$ are the constants that depend on the distance between the initial point and the optimal point.
\end{lemma}

% \begin{proof}
% As $L$ is linear in $\lambda$ and $[0,\Lambda]^{i-1}$ is a bounded and convex set, Eq.~(\ref{eq:auxiliary-lagrange}) is a nonconvex-concave minimax problem. By the Assumption~\ref{assump2} and \cite{lin2020gradient} (Thm. 4.8), we know that using the Two-Timescale GDA with the learning rate $\eta_x = \Theta(\epsilon^4/l^3 \zeta^2 \Lambda^2)$ and $\eta_y = \Theta(1/l)$, the algorithm converge to an $\epsilon$-stationary point $\pi_i$ and the number of required iterations is 
% $$\mathcal{O}\left(\frac{l^3 \zeta^2 \Lambda^2 C_1}{\epsilon^6} + \frac{l^3 \Lambda^2 C_2}{\epsilon^4}\right),$$
% \end{proof}

%%%%%%%%%%%%%%%%%%%%%%%%%%%%%%%%%%%%%%%%%%%%%%%%
%%% thm
% Then we introduce the definition of $\epsilon$-approximate KKT point:

% \textbf{$\epsilon$-KKT point}
% A point $\pi$ which is feasible to Eq.~(\ref{eq:iter-co}) is said to be an $\epsilon$-approximate KKT point if given $\epsilon > 0$,  .

\textbf{Theorem~\ref{thm:converge}.}
\textit{
Under assumptions~\ref{assump1} and ~\ref{assump2} and learning rate with learning rate $\eta_\pi = \Theta(\epsilon^4/l^3 \zeta^2 \Lambda^2)$ and $\eta_{\bm{\lambda}} = \Theta(1/l)$, {\name} converges to an $\epsilon$-stationary point with convergence rate $\mathcal{O}\left(\frac{l^3 \zeta^2 \Lambda^2 C_1}{\epsilon^6} + \frac{l^3 \Lambda^2 C_2}{\epsilon^4}\right)$.
}

\begin{proof}
We consider the following constraint nonconvex-concave optimization:
\begin{align}
    \min_{\pi} \max_{0\le\bm{\lambda}\le\Lambda} L(\pi, \bm{\lambda})\,.
\end{align}

Following Lemma \ref{lem4},
we know that the Two-Timescale GDA algorithm converges to an $\epsilon$-stationary point $\pi_i^*$. 
% Denote $\Phi(\pi) = \max_{\bm{0\le \lambda}\le\Lambda}  L(\pi, \bm{\lambda})$ and $\|\cdot\|$ as the Euclidean distance. Using the property of $\epsilon$-stationary point $\pi_i^0$ in \cite{lin2020gradient} (Lemma 3.8), we know that there exists $\hat{\pi}_i$ such that $\min_{\xi \in \partial \Phi(\hat{\pi}_i) } \| \xi \| \le \epsilon $ and $ \| \hat{\pi}_i - \pi_i^0 \|  \le \epsilon / 2l .$

% From the definition of $L(\pi_i, \bm{\lambda})$, we know that 
% $\hat{\pi}_i$ is an $\epsilon$-approximate KKT point of $J(\pi)$(\cite{dutta2013approximate}).

From the above deduction, the Two-Timescale GDA algorithm requires $\mathcal{O}\left(\frac{l^3 \zeta^2 \Lambda^2 C_1}{\epsilon^6} + \frac{l^3 \Lambda^2 C_2}{\epsilon^4}\right)$ iterations with learning rate $\eta_\pi = \Theta(\epsilon^4/l^3 \zeta^2 \Lambda^2)$ and $\eta_{\bm{\lambda}} = \Theta(1/l)$ to converge to an $\epsilon$-stationary point with convergence rate.

% The theorem then follows by applying the smoothness assumption.

\end{proof}

\input{A5-discussion.tex}

\section{Pseudocode of {\name}}
\label{app:algorithmic}

% FOR the $k$-th iteration, the SIPO algorithm will yield the $k$-th policy $\pi_k$.
The pseudocode of {\name} is shown in Algorithm~\ref{algo:SIPO}.

\begin{algorithm}
    \caption{{\name} (\textcolor{red}{red} for \textcolor{red}{{\namerbf}} and \textcolor{blue}{blue} for \textcolor{blue}{{\namewd}})}  
    \label{algo:SIPO}
\begin{algorithmic}[1]
\renewcommand\algorithmicrequire{\textbf{Input:}}
  \REQUIRE{Number of Iterations $M$, Number of Training Steps within Each Iteration $T$.
  }
\renewcommand\algorithmicrequire{\textbf{Hyperparameter:}}
\REQUIRE{Learning Rate $\eta_\pi$, Diversity Threshold $\delta$, Intrinsic Scale Factor $\alpha$, 
  Lagrange Multiplier Upperbound $\lambda_\textrm{max}$, Lagrange Learning rate $\eta_\lambda$, \textcolor{blue}{Wasserstein Critic Learning Rate $\eta_W$}, \textcolor{red}{RBF Kernel Variance $\sigma$}.
  }
  \renewcommand\algorithmiccomment[1]{\hfill // {#1}}
  \STATE{Archived trajectories $X\gets\emptyset$}\COMMENT{to store states visited by previous policies}
  \FOR{iteration $i=1,\dots,M$}
    \STATE{Initialize policy $\pi_{\theta_i}$}\COMMENT{initialization}
    \STATE{\textcolor{blue}{Initialize Wasserstein critic $f_{\phi_i}$}}
    \FOR{archive index $j=1,\dots,i-1$}
      \STATE{Lagrange multiplier $\lambda_j\gets0$}
    \ENDFOR
    \FOR{Training step $t=1,\dots,T$}
      \STATE{Collect trajectory $\tau=\left\{\left(s_h,\bm{a}_h,r(s_h,\bm{a}_h)\right)\right\}_{h=1}^H$}
      \FOR{archive index $j=1,\dots,i-1$}
        \STATE{$R^j_\textrm{int}\gets0$}
      \ENDFOR
      \FOR{timestep $h=1,\dots,H$}
        \STATE{$r_{\textrm{int},h} \leftarrow 0$}\COMMENT{compute intrinsic reward}
        \FOR{archive trajectory $\chi_j \in X$}
          \STATE{\textcolor{red}{$r^j_{\textrm{int},h} \gets  - \frac{1}{H\vert\chi_j\vert}\sum_{s^\prime\in\chi_j}\exp\left(-\frac{\Vert s_h-s^\prime\Vert^2}{2\sigma^2}\right)$}}
          \STATE{\textcolor{blue}{$r^j_{\textrm{int},h} \gets  \frac{1}{H}\left[f_{\phi_j}(s_h)-\frac{1}{\vert\chi_j\vert}\sum_{s^\prime\in\chi_j} f_{\phi_j}(s^\prime)\right]$}}
          \STATE{$r_{\textrm{int},h} \gets r_{\textrm{int},h} + \lambda_j\cdot r^j_{\textrm{int},h}$}
          \STATE{$R^j_{\textrm{int}} \gets R^j_{\textrm{int},h} + r^j_{\textrm{int},h}$}
        \ENDFOR
        \STATE{$r_h\gets r(s_h,\bm{a}_h)+\alpha\cdot r_{\textrm{int},h}$}
      \ENDFOR
      \FOR{archive index $j=1,\dots,i-1$}
        \STATE{$\lambda_j \gets \textrm{clip}\left(\lambda_j +\eta_\lambda \left(- R^j_{\textrm{int}} + \delta\right), \,0, \,\lambda_\textrm{max}\right)$}\COMMENT{gradient ascent on $\lambda_j$}
        \STATE{\textcolor{blue}{
        $\phi_j\gets\phi_j+\eta_W\frac{1}{H}\sum_{h=1}^H \nabla_{\phi_j}\left(
        f_{\phi_j}(s_h)-\frac{1}{\vert\chi_j\vert}\sum_{s^\prime\in\chi_j} f_{\phi_j}(s^\prime)
        \right)$
        }}
        \STATE{\textcolor{blue}{$\phi_j\gets\textrm{clip}(\phi_j, -0.01, 0.01)$}}
        % \COMMENT{\textcolor{blue}{update Wasserstein critic under Lipchitz constraint}}
      \ENDFOR
      \STATE{Update $\pi_{\theta_i}$ with $\{(s_h,\bm{a}_h,r_h)\}$ by PPO algorithm}\COMMENT{policy gradient on $\theta_i$}
    \ENDFOR
    \STATE{Collect many trajectories $\chi_i$}\COMMENT{collect trajectories to approximate $d_{\pi_{\theta_i}}$}
    \STATE{$X\gets X\cup \{\chi_i\}$}\COMMENT{for the use of following iterations}
  \ENDFOR
\end{algorithmic}  
\end{algorithm}

\end{document}

%% file: math_commands.tex
\usepackage{amsmath,amsfonts,bm}

\def\eqref#1{equation~\ref{#1}}
\def\1{\bm{1}}

\DeclareMathAlphabet{\mathsfit}{\encodingdefault}{\sfdefault}{m}{sl}
\SetMathAlphabet{\mathsfit}{bold}{\encodingdefault}{\sfdefault}{bx}{n}

\def\gA{{\mathcal{A}}}

\def\gM{{\mathcal{M}}}

\def\gO{{\mathcal{O}}}

\def\gS{{\mathcal{S}}}

\def\sR{{\mathbb{R}}}

\newcommand{\E}{\mathbb{E}}

\newcommand{\R}{\mathbb{R}}

\DeclareMathOperator*{\argmax}{arg\,max}

%% file: 1-intro.tex
\section{Introduction}

A consensus in deep learning (DL) is that different local optima have similar mappings in the functional space, leading to similar losses to the global optimum~\cite{venturi2018novalley,roughgarden2020beyond,ma_why_2021}. 
Hence, via stochastic gradient descent (SGD), most DL works only focus on the final performance without considering \emph{which} local optimum SGD discovers.
However, in complex reinforcement learning (RL) problems, the policies associated with different local optima can exhibit significantly different behaviors~\cite{boatracing,liu2021motor,vinyals_grandmaster_2019}.
Thus, it is a fundamental problem for an RL algorithm to not only optimize rewards but also discover as many diverse strategies as possible.
A pool of diversified policies can be further leveraged towards a wide range of applications, including the discovery of emergent behaviors~\citep{liu2019emergent, baker_emergent_2020,tang_discovering_2021}, generating diverse dialogues~\citep{dialogue}, designing robust robots~\citep{cully_robots_2015,kumar_one_2020,gupta2021embodied}, and enhancing human-AI collaboration~\citep{lupu_trajectory_2021,charakorn2022incompatible,cui2023adversarialhanani}.

Obtaining diverse RL strategies requires a quantitative method for measuring the difference (i.e., \emph{diversity}) between two policies. However, how to define such a measure remains an open challenge.
Previous studies have proposed various diversity measures, such as comparing the difference between the action distributions generated by policies~\cite{sun2020novel,lupu_trajectory_2021,zhou_continuously_2022}, computing probabilistic distances between the state occupancy of different policies~\cite{masood2019diversity}, or measuring the mutual information between states and policy identities~\cite{eysenbach_diversity_2018}. However, it remains unclear which measure could produce the best empirical performance. Besides, the potential pitfalls of these measures are rarely discussed. % unclear which measure to choose and whether certain measures may have potential pitfalls in specific applications.

In addition to diversity measures, there are two common computation frameworks for discovering diverse policies, including population-based training (PBT) and iterative learning (ITR). PBT directly solves a constrained optimization problem by learning a collection of policies simultaneously, subject to policy diversity constraints~\citep{parker-holder_effective_2020,lupu_trajectory_2021,charakorn2022incompatible}. Although PBT is perhaps the most popular framework in the existing literature, it can be computationally challenging~\citep{omidshafiei2020navigating} since the number of constraints grows quadratically with the number of policies.
The alternative framework is ITR, which iteratively learns a single policy that is sufficiently different from previous policies~\citep{masood2019diversity,zhou_continuously_2022}.
ITR is a greedy relaxation of the PBT framework and it largely simplifies the optimization problem in each iteration.
However, the performance of the ITR framework has not been theoretically analyzed yet, and it is often believed that ITR can be less efficient due to its sequential nature.

We provide a comprehensive study of the two aforementioned design choices. First, we examine the limitations of existing diversity measures in a few representative scenarios, where two policies outputting very different action distributions can still lead to similar state transitions. %Our findings reveal that actions with high diversity scores can lead to similar state transitions. 
In these scenarios, state-occupancy-based measures are not sufficient to truly reflect the underlying behavior differences of the policies either.
%Furthermore, state occupancy measures are insufficient in capturing the behavioral differences in states. 
By contrast, we observe that diversity measures based on \emph{state distances} can accurately capture the visual behavior differences of different policies. %meaningful differences in behavior. 
Therefore, we suggest that an effective diversity measure should explicitly incorporate state distance information for the best practical use.
Furthermore, for the choice of computation framework, we conduct an in-depth analysis of PBT and ITR. We provide theoretical evidence that ITR, which has a simplified optimization process with fewer constraints, can discover solutions with the same reward as PBT while achieving \emph{at least half} of the diversity score.
This finding implies that although ITR is a greedy relaxation of PBT, their optimal solutions can indeed have comparable qualities.
Furthermore, note that policy optimization is much simplified in ITR, which suggests that ITR can result in much better empirical performances and should be preferred in practice.

Following our insights, we combine ITR and a state-distance-based diversity measure to develop a generic and effective algorithm, \emph{State-based Intrinsic-reward Policy Optimization (\name)}, for discovering diverse RL strategies.
In each iteration, we further solve this constrained optimization problem via the Lagrange method and two-timescale gradient descent ascent (GDA)~\citep{lin2020gradient}.
{We theoretically prove that} %Theoretical results show that 
our algorithm is guaranteed to converge to a neighbor of $\epsilon$-stationary point.
Regarding the diversity measure, we provide two practical realizations, including a straightforward version based on the RBF kernel and a more general learning-based variant using Wasserstein distance.

We evaluate {\name} in three domains ranging from single-agent continuous control to multi-agent games: Humanoid locomotion~\cite{makoviychuk2021isaac}, StarCraft Multi-Agent Challenge~\citep{samvelyan2019starcraft}, and Google Research Football (GRF)~\citep{kurach2020google}.
Our findings demonstrate that {\name} surpasses baselines in terms of population diversity score across all three domains.
Remarkably, our algorithm can successfully discover 6 distinct human-interpretable strategies in the GRF 3-vs-1 scenario and 4 strategies in two 11-player GRF  scenarios, namely counter-attack, and corner, without any domain-specific priors, which are beyond the capabilities of existing algorithms.

%% file: 2-related.tex
\section{Related Work}

\textbf{Diversity in RL.}
It has been shown that policies trained under the same reward function can exhibit significantly different behaviors~\citep{boatracing,liu2021motor}. Merely discovering a single high-performing solution may not suffice in various applications~\citep{cully_robots_2015,vinyals_grandmaster_2019,kumar_one_2020}.
As such, the discovery of a diverse range of policies is a fundamental research problem, garnering attention over many years~\citep{miller_genetic_1996,deb_finding_2010,lee_diversify_2022}.
Early works are primarily based on multi-objective optimization~\citep{mouret_illuminating_2015,pugh_quality_2016,ma_efficient_2020,nilsson_policy_2021,pierrot_diversity_2022}, which assumes a set of reward functions is given in advance. In RL, this is also related to reward shaping~\citep{ng1999policy,babes_social_2008,devlin_theoretical_2011,tang_discovering_2021}. We consider learning diverse policies without any domain knowledge.

\textbf{Population-based training (PBT)} is the most popular framework for diverse solutions by jointly learning separate policies. Representative works include evolutionary computation~\cite{wang2019poet,long2019evolutionary,parker-holder_effective_2020}, league training~\citep{vinyals_grandmaster_2019,jaderberg_human-level_2019}, computing Hessian matrix~\cite{parker-holder_ridge_2020}  or constrained optimization with a population diversity measure~\cite{lupu_trajectory_2021,zhao2021maximum,li_celebrating_2021,liu_unifying_2021,charakorn2022incompatible}.
An improvement is to learn a latent variable policy instead of separate ones.
% to improve sample efficiency. 
Prior works have incorporated different domain knowledge to design the latent code, such as action clustering~\citep{wang_rode_2020}, agent identities~\citep{li_celebrating_2021,marlrolediag} or prosocial level~\citep{peysakhovich_consequentialist_2018,baker_emergent_2020}.
The latent variable can be also learned in an unsupervised fashion, such as in DIYAN~\citep{eysenbach_diversity_2018} and its variants~\citep{kumar_one_2020,osa_discovering_2022}.
\citet{domino} learns diverse policies with hard constraints on rewards to ensure the derived policies are (nearly) optimal, potentially hindering policies with disparate reward scales. On the other hand, our method prioritizes diversity and fully accepts sub-optimal strategies.

\textbf{Iterative learning ({ITR})} simplifies PBT by only optimizing a single policy in each iteration and forcing it to behave differently w.r.t. previously learned ones~\cite{masood2019diversity,sun2020novel,zhou_continuously_2022}.
While some ITR works require an expensive clustering process before each iteration~\cite{zhang2019novel_task} or domain-specific features~\cite{zahavy2021discovering}, we consider domain-agnostic ITR in an end-to-end fashion.
Besides, \citet{pacchiano2020learning_to_score} learns a kernel-based score function to iteratively guide policy optimization. The score function is conceptually similar to {\namewd} but is applied to a parallel setting with more restricted expressiveness power.

\textbf{Diversity Measure.}
Most previous works considered diversity measures on action distribution and state occupancy. For example, measures such as Jensen-Shannon divergence~\cite{lupu_trajectory_2021} and cross-entropy~\cite{zhou_continuously_2022} are defined over policy distributions to encourage different policies to take different actions on the same state, implicitly promoting the generation of diverse trajectories. Other measures such as maximum mean discrepancy~\cite{masood2019diversity} maximize the probability distance between the state distributions induced by two policies. However, these approaches can fail to capture meaningful behavior differences between two policies in certain scenarios, %have limitations in capturing behavioral differences, 
as we will discuss in Section~\ref{sec:case:measure}. 
There also exist specialized measures, such as cross-play rewards~\cite{charakorn2022incompatible}, which are designed for cooperative multi-agent games. 
It is worth noting that diversity measures are closely related to exploration criteria~\cite{bellemare2016countexploration,hazan2019maxentexpl,burda2018exploration,state-marginal-matching} and skill discovery~\cite{campos2020explore,liu2021behavior,jiang2022unsupervised}, where a diversity surrogate objective is often introduced to encourage broad state coverage.
However, this paper aims to explicitly discover mutually distinct policies. Our diversity measure depends on a function that computes the distance between states visited by two policies.

%% file: 3-prelim.tex
\vspace{-1.5mm}
\section{Preliminary}
\vspace{-1.5mm}

\textbf{Notation:}
We consider POMDP~\citep{spaan2012partially} defined by $\gM = \langle \gS, \gA, \gO, r, P, O, \nu, H\rangle$.
$\gS$ is the state space.
$\gA$ and $\gO$ are the action and observation space.
$r : \gS \times \gA \to \sR$ is the reward function.
$O: \gS \to \gO$ is the observation function.  $H$ is the horizon.
$P$ is the transition function.
% For state $s, s' \in \gS$ and an action $a \in \gA$,
% the transition probability from $s$ to $s'$
% by executing action $a$ is $P(s' \mid s, a)$.
At timestep $h$, the agent
receives an observation $o_h=O(s_h)$ and outputs an action $a_h\in\gA$ w.r.t. its policy $\pi : \gO \to \triangle\left(\gA\right)$.
The RL objective $J(\pi)$ is defined by
$
J(\pi)=\mathbb{E}_{(s_h,a_h)\sim(P,\pi)}\left[\sum_{h=1}^H r(s_h,a_h)\right].
$
% The discounted factor is omitted here to simplify notations.
The above formulation can be naturally extended to cooperative multi-agent settings, where $\pi$ and $R$ correspond to the joint policy and the shared reward.
We follow the standard POMDP notations for conciseness. Our method will be evaluated in both single-agent tasks and complex cooperative multi-agent scenarios. Among them, multi-agent environments encompass a notably more diverse range of potential winning strategies, and hence offer an apt platform for assessing the effectiveness of our method.
Moreover, in this paper, we \textbf{assume access to object-centric information and features} rather than pure visual observations to simplify our discussion. {We remark that although we restrict the scope of this paper to states, our method can be further extended to high-dimensional inputs (e.g. images, see App.~\ref{app:ablation-state-input}) or tabular MDPs via representation learning~\cite{laplacian,prototypical-repr}}.

Finally, to discover diverse strategies, 
we aim to learn a set of $M$ policies $\{\pi_i\}_{i=1}^M$ such that all of these policies are locally optimal under $J(\cdot)$ but mutually distinct subject to some diversity measure $D(\cdot,\cdot): \triangle\times\triangle\to\R$, which captures the difference between two policies.

\textbf{Existing Diversity Measures:}
We say a diversity measure $D$ is defined over action distribution if it can be written as

\vspace{-4mm}
\begin{equation}\label{eq:measure-action}
    D (\pi_i,\pi_j)=\E_{s\sim q(s)}\left[\tilde{D}_\gA \left(\pi_i(\cdot \mid s)\Vert \pi_j(\cdot \mid s)\right)\right],
\end{equation}
\vspace{-2mm}

where $q$ is an occupancy measure over states,
% denotes some specific state distribution (e.g. the joint state distribution of $\pi_i$ and $\pi_j$), 
$\tilde{D}_\gA:\triangle\times\triangle\to\R$ measures the difference between action distributions.
% $q:\triangle(\gS)$ is a proposal distribution of states.
$\tilde{D}_\gA$ can be any probability distance as defined in prior works~\cite{sun2020novel,lupu_trajectory_2021,zhou_continuously_2022,parker-holder_effective_2020}.

Denote the state occupancy of $\pi$ as $q_\pi$. We say a diversity measure is defined over state occupancy if it can be written as

\vspace{-4mm}
\begin{equation}
D(\pi_i,\pi_j)=\tilde{D}_\gS \left( q_{\pi_i} \Vert q_{\pi_j} \right),
% \left(
% \E_{(s,s^\prime)\sim \gamma}\left[ g\left(d\left( s,s^\prime\right)\right) \right],
% \right)^\frac{1}{2}
\label{eq:state-occup}
\end{equation}
\vspace{-2mm}

which can be realized as an integral probability metric~\cite{masood2019diversity}.
We remark that $q_\pi$ is usually intractable.  % in practice.

In addition to diversity measures, we present two popular computation frameworks %procedures 
for this purpose.

\textbf{Population-Based Training (PBT):}
PBT is a straightforward formulation by jointly learning $M$ policies $\{\pi_i\}_{i=1}^M$ subject to pairwise diversity constraints, i.e.,

\vspace{-4mm}
\begin{footnotesize}\begin{align}
    \max_{\{\pi_i\}}\, \sum_{i=1}^M J(\pi_i) &
        \;\;\textrm{s.t.} \; D(\pi_j,\pi_k)\ge\delta,\forall j,k\in [M],\,j\neq k,
        \label{eq:pbt-co}
\end{align}\end{footnotesize}
\vspace{-2mm}

where $\delta$ is a threshold. In our paper, we consistently refer to the aforementioned computation framework as "PBT", rather than adjusting hyperparameters~\citep{jaderberg_population_2017}. Despite a precise formulation, PBT poses severe optimization challenges due to mutual constraints.

% \outline{iterative learning}

\textbf{Iterative Learning (ITR):} {ITR} is a greedy approximation of PBT by iteratively learning novel policies.
In the $i$-th ($1\le i\le M$) iteration, {ITR} solves
\vspace{-4mm}

\begin{footnotesize}
\begin{align}
        \pi_i^\star =\arg\max_{\pi_i} J(\pi_i)&
        \;\;\textrm{s.t.} \; D(\pi_i,\pi_j^\star)\ge\delta,\forall 1\le j< i.
        \label{eq:iter-co}
\end{align}
\end{footnotesize}

\vspace{-2mm}
$\pi_j^\star$ is recursively defined by the above equation. Compared with PBT, ITR trades off wall-clock time for less required computation resources (e.g., GPU memory) and performs open-ended training (i.e., the population size $M$ does not need to be fixed at the beginning of training).

%% file: 4-analysis-component.tex
\section{Analysis of Existing Diversity-Discovery Approaches}
\label{sec:analysis-component}
In this section, we conduct both quantitative and theoretical analyses of existing approaches to motivate our method.
We first discuss diversity measures in Sec.~\ref{sec:case:measure} and then compare computation frameworks, namely PBT and {ITR}, in Sec.~\ref{sec:case:framework}.

\subsection{A Common Missing Piece in Diversity Measure: State Distance}
\label{sec:case:measure}

\begin{figure}[tb]
\hspace{3mm}
\begin{minipage}{0.18\textwidth}
\centering
\includegraphics[width=\textwidth]{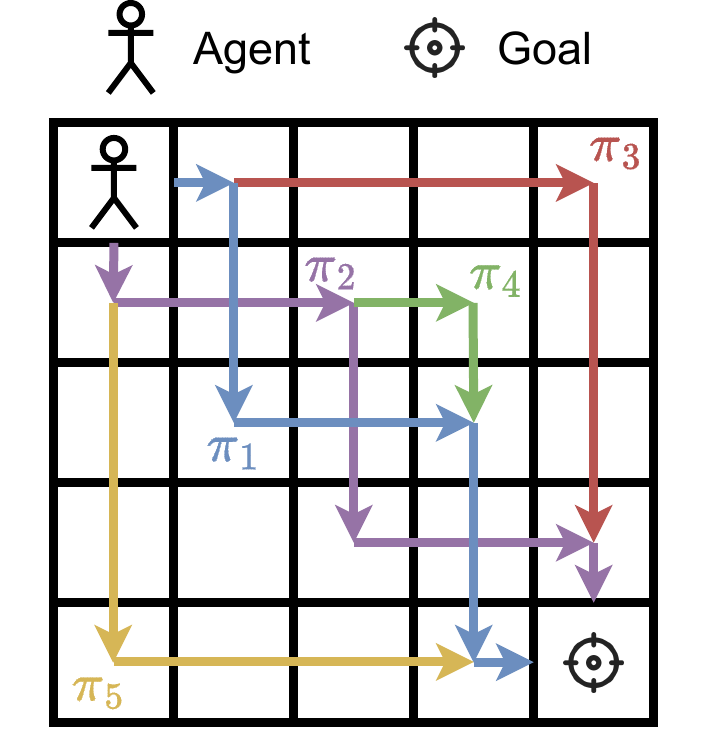}
\end{minipage}
\hfill
\begin{minipage}{0.77\textwidth}
\centering
\scriptsize
\captionof{table}{\small{Diversity measures of the grid-world example.
Computation details can be found in App.~\ref{app:addtional}. 
% (KL=KL divergence, $\textrm{JSD}_\gamma$= generalized Jensen-Shannon Divergence, EMD=Earth Moving Distance)
}}
\vspace{1mm}
\begin{tabular}{cccccccc}
\toprule  
 & \multirow{2}{*}{human}&\multicolumn{4}{c}{action-based} & \multicolumn{2}{c}{state-distance-based}\\
\cmidrule(lr){3-6} \cmidrule(lr){7-8}
 & & KL & $\textrm{JSD}_1$ & $\textrm{JSD}_0$/EMD & $L_2$ norm & $L_2$ norm & EMD\\
\midrule
$D(\pi_1,\pi_2)$ & small & $\bm{+\infty}$ & $\bm{\log 2}$ & $\bm{1/2}$  & $\bm{\sqrt{7}}$ & $2\sqrt{2}$ & $5.7$ \\
$D(\pi_1,\pi_3)$ & large & $\bm{+\infty}$ & $\bm{\log 2}$ &  $1/8$ & $1$ & $\bm{2\sqrt{6}}$ & $\bm{11.3}$\\  \bottomrule
\end{tabular}
\label{tab:novel-measure}
\end{minipage}
\captionof{figure}{\small{(left) A grid-world environment with 5 different optimal policies.
Intuitively, $D(\pi_1,\pi_2)<D(\pi_1,\pi_3)$ and $D(\pi_3,\pi_4)<D(\pi_3,\pi_5)$.
However, action-based measures can give $D_\gA(\pi_1,\pi_2)\ge D_\gA(\pi_1,\pi_3)$ and state-occupancy-based measures can give $D(\pi_3,\pi_4)=D(\pi_3,\pi_5)$.}}
\label{fig:gw-env}
\vspace{-3mm}
\end{figure}

The perception of diversity among humans primarily relies on the level of dissimilarity within the state space, which is measured by a distance function.
However, the diversity measures outlined in Eq.~(\ref{eq:measure-action}) and Eq.~(\ref{eq:state-occup}) completely fail to account for such crucial information.
In this section, we provide a detailed analysis to instantiate this observation with concrete examples and propose a novel diversity measure defined over state distances.

First, we present a synthetic example to demonstrate the limitations of current diversity measures. Our example consists of a grid-world environment with a single agent and grid size $N_G$. The agent starts at the top left of the grid-world and must navigate to the bottom right corner, as shown in Fig.~\ref{fig:gw-env}. While $N_G$ can be large in general, we illustrate with $N_G=5$ for simplicity. We draw five distinct policies, denoted as $\pi_1$ through $\pi_5$, which differ in their approach to navigating the grid-world.
Consider $\pi_1$, $\pi_2$, and $\pi_3$ first.
Although humans may intuitively perceive that policies $\pi_1$ and $\pi_2$, which move along the diagonal, are more similar to each other than to $\pi_3$, which moves along the boundary, diversity measures based on actions can fail to reflect this intuition, as shown in Table~\ref{tab:novel-measure}.
Then, let's switch to policies $\pi_3$, $\pi_4$, and $\pi_5$. We find that state-occupancy-based diversity measures are unable to differentiate between $\pi_4$ and $\pi_5$ in contrast to $\pi_3$. This is because the states visited by $\pi_3$ are entirely disjoint from those visited by both $\pi_4$ and $\pi_5$. However, humans would judge $\pi_5$ to be more distinct from $\pi_3$ than $\pi_4$ because both $\pi_3$ and $\pi_4$ tend to visit the upper boundary.

Next, we consider a more realistic and complicated multi-agent football scenario, i.e., the Google Research Football~\cite{kurach2020google} environment, in Fig.~\ref{fig:gw-fb}, where an idle player in the backyard takes an arbitrary action without involving in the attack at all. 
Although the idle player stays still with no effect on the team strategy, action-based measures can produce high diversity scores.
This example underscores a notable issue. If action-based measures are leveraged to optimize diversity, the resultant policies can produce visually similar behavior.
While it can be possible to exclude idle actions by modifying task rewards, it requires domain-specific hacks and engineering efforts. The issue of idle actions exists even in such popular MARL benchmarks. Similar issues have also been observed in previous works~\cite{lupu_trajectory_2021}.

To summarize, existing measures suffer from a significant limitation --- they only compare the behavior trajectories \emph{implicitly} through the lens of action or state distribution without \emph{explicitly measuring state distance}. Specifically, action-based measures fail to capture the behavioral differences that may arise when similar states are reached via different actions. Similarly, state occupancy measures do not quantify \emph{the degree of dissimilarity} between states. To address this limitation, we propose a new diversity measure that explicitly takes into account the distance function in state space:
\begin{equation}
D(\pi_i,\pi_j)=
% \left(
\E_{(s,s^\prime)\sim \gamma}\left[ g\left(d\left( s,s^\prime\right)\right) \right],
% \right)^\frac{1}{2}
\label{eq:state-dist}
\end{equation}
$d$ is a distance metric over $\gS\times\gS$. $g:\R^+\to\R$ is a monotonic cost function. $\gamma\in\Gamma(q_{\pi_i},q_{\pi_j})$ is a distribution over state pairs. $\Gamma(q_{\pi_i},q_{\pi_j})$ denotes the collection of all distributions on $\gS\times\gS$ with marginals $q_{\pi_i}$ and $q_{\pi_j}$ on the first and second factors respectively.
The cost function $g$ is a notation providing a generalized and unified definition. It also contributes to training stability by scaling the raw distance.
We highlight that Eq.~(\ref{eq:state-dist}) computes the cost on individual states before taking expectation, and therefore prevents information loss of taking the average over the entire trajectory (e.g. the DvD score~\cite{parker-holder_effective_2020}).
We also note that states are consequences of performed actions. Hence, a state-distance-based measure also implicitly reflects the (meaningful) differences in actions between two policies.
We compute two simple measures based on state distance, i.e., the $L_2$ norm and the Earth Moving Distance (EMD), for the grid-world example and present results in Table~\ref{tab:novel-measure}. These measures are consistent with human intuition.

\begin{figure*}[t]
\begin{minipage}[c]{0.3\textwidth}
        \centering
        \includegraphics[width=0.82\textwidth]{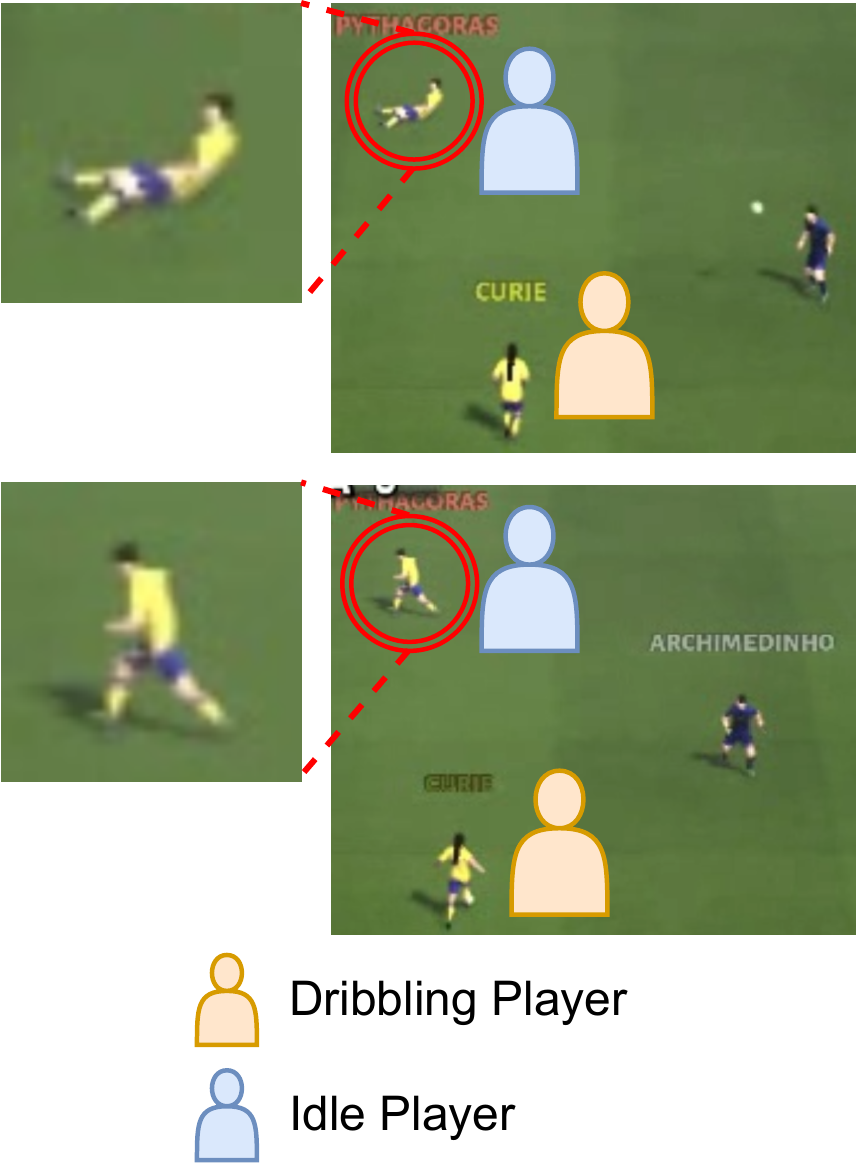}
        % \vspace{-2mm}
        \caption{\small{
        A counter-example for action-based diversity measure: in a football game, we can achieve a high diversity score by simply asking a single idle player to output random actions, which does not affect the high-level gameplay strategy at all. 
        }}
        \label{fig:gw-fb}
    \end{minipage}
    \hfill
    \begin{minipage}[c]{0.68\textwidth}
    \vspace{-3mm}
        \centering
        \includegraphics[width=\textwidth]{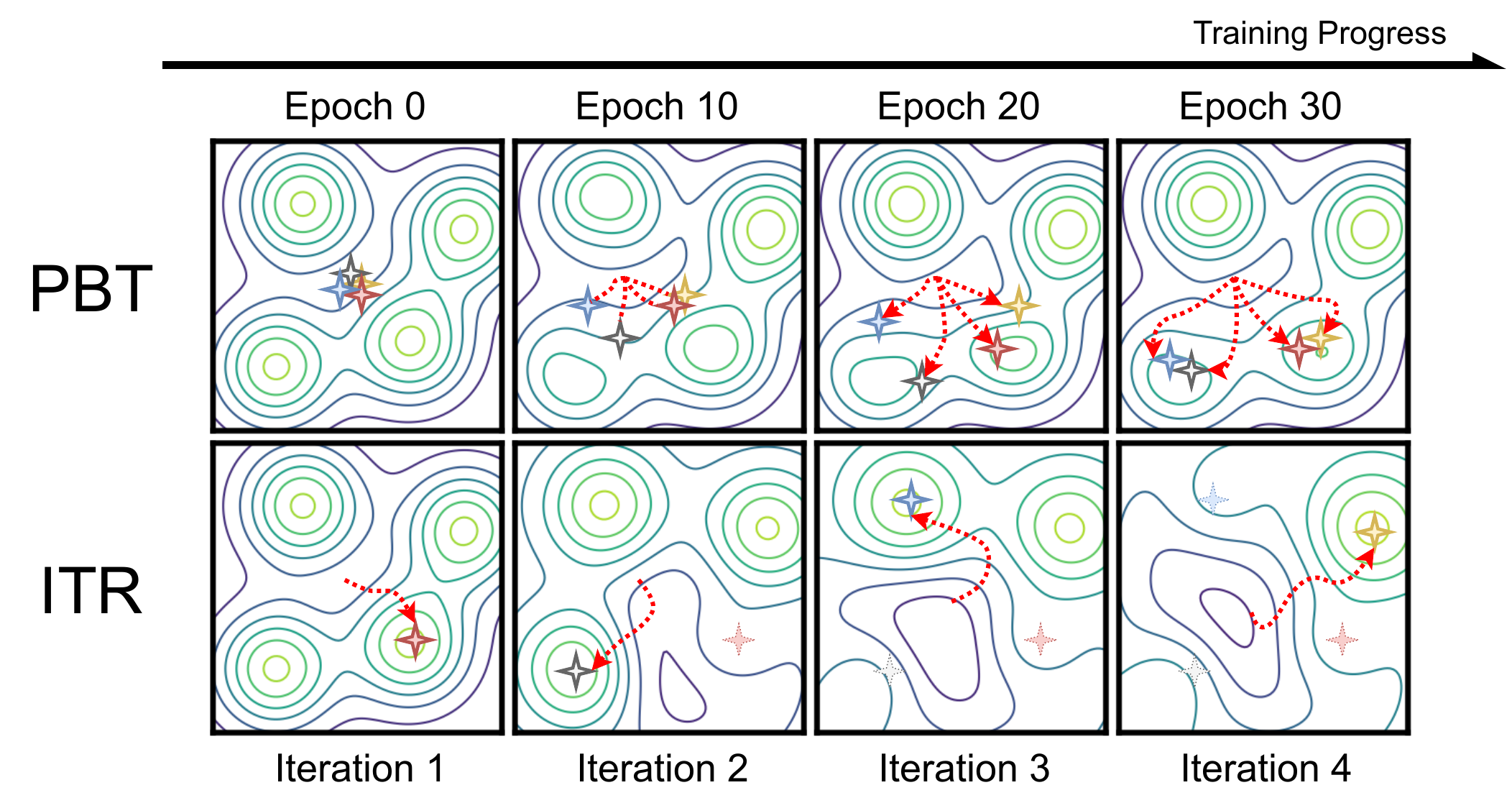}
        \vspace{-3mm}
        \caption{
        Illustration of the learning process of PBT and {ITR} in a 2-D navigation environment with 4 modes.
        PBT will not uniformly converge to different landmarks as computation can be either too costly or unstable.
        By contrast, {ITR} repeatedly excludes a particular landmark, such that policy in the next iteration can continuously explore until a novel landmark is discovered.
        }
        \label{fig:pbt-vs-iter}
    \end{minipage}
    \vspace{-3mm}
\end{figure*}

\subsection{Computation Framework: Population-Based or Iterative Learning? }
\label{sec:case:framework}

We first consider the simplest motivating example to intuitively illustrate the optimization challenges. Let's assume that $\pi_i$ is a scalar, $J(\pi_i)$ is linear in $\pi_i$, and $D(\pi_i,\pi_j)=\vert\pi_i-\pi_j\vert$. In our definition, where $M$ denotes the number of diverse policies, PBT involves $\Theta(M^2)$ constraints in a single linear programming problem while {ITR} involves $\gO(M)$ constraints in each of $M$ iterations.
{Given that the complexity of linear programming is a high-degree polynomial (higher than 2) of the number of constraints, solving PBT is harder (and probably slower) than solving ITR in a total of $M$ iterations, \emph{despite PBT being parallelized}.}
This challenge can be more severe in RL due to complex solution space and large training variance.

\begin{wrapfigure}{r}{0.5\textwidth}
\centering
    \vspace{-8mm}
    \includegraphics[width=0.5\textwidth]{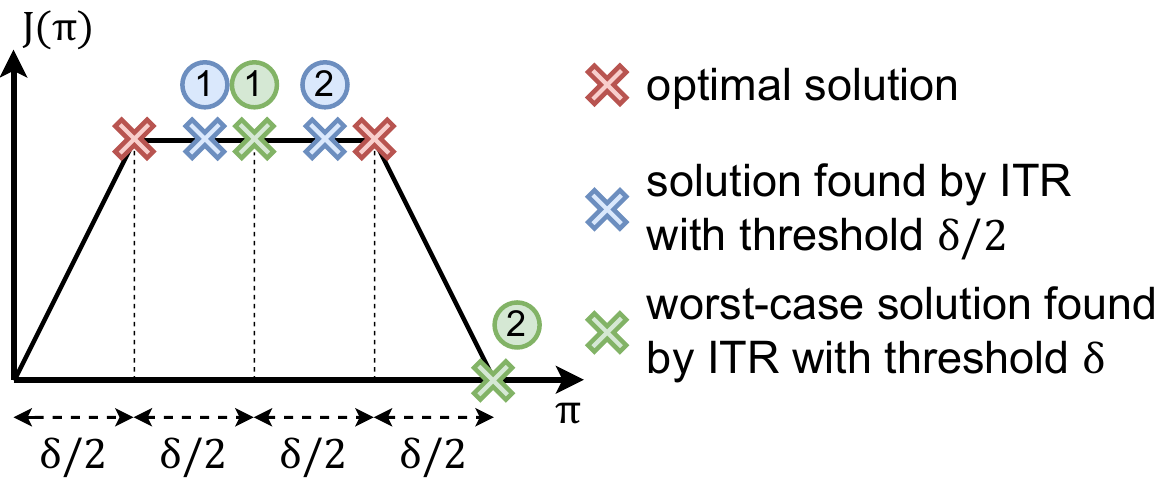}
    % \vspace{-3mm}
    \caption{\small{1-D worst case of ITR. With threshold $\delta$,
    ITR finds solutions with inferior rewards. However, ITR can find optimal solutions
    if the threshold is halved.}}
    \label{fig:thm1wc}
    \vspace{-5mm}
\end{wrapfigure}
Although {ITR} can be optimized efficiently,
it remains unclear whether {ITR}, as a greedy approximation of PBT, can obtain solutions of comparable rewards.
Fig.~\ref{fig:thm1wc} shows the worst case in the 1-D setting when the {ITR} solutions (green)
can indeed have lower rewards than the PBT solution (red) subject to the same diversity constraint.
However, we will show in the next theorem that {ITR} is guaranteed to have no worse rewards than PBT by trading off half of the diversity.

\vspace{1mm}
\begin{theorem}
Assume $D$ is a distance metric. Denote the optimal value of Eq.(~\ref{eq:pbt-co}) as $T_1$.
Let $T_2=\sum_{i=1}^M J(\tilde{\pi}_i)$ where
\begin{footnotesize}
\begin{align}
\tilde{\pi}_i =\arg\max_{\pi_i}J(\pi_i)&\;\;
\textrm{s.t.} \;D(\pi_i,\tilde{\pi}_j)\ge\delta/2,\forall 1\le j< i
\label{eq:iter-thm1}
\end{align}
\end{footnotesize}
for $i=1,\dots,M$, then $T_2\ge T_1$.
\label{thm:2delta}
\end{theorem}

Please see App.~\ref{app:proof1} for the proof.
The above theorem provides a quality guarantee for {ITR}. The proof can be intuitively explained by the 1-D example in Fig.~\ref{fig:thm1wc},
{where green points represent the worst case with threshold $\delta$ and blue points represent the solutions with threshold $\delta/2$.}
Thm.~\ref{thm:2delta} shows that, for any policy pool derived by PBT, we can always use {ITR} to obtain another policy pool, which has \emph{the same rewards} and \emph{comparable diversity scores}.

\begin{wraptable}{r}{4.5cm}
\vspace{-5mm}
\footnotesize
\centering
\caption{\small{The number of discovered landmarks across 6 seeds with standard deviation in the bracket.}}
\label{tab:pbt-vs-iter}
\vspace{1mm}
\begin{tabular}{ccc}
\toprule  
setting & PBT & {ITR} \\
\midrule
$N_L=4$ & 2.0 (1.0)& \textbf{3.5} (0.5) \\ 
$N_L=5$ & 2.2 (0.9)& \textbf{4.5} (0.5) \\  \bottomrule
\end{tabular}
\vspace{-3mm}
\end{wraptable}
\textbf{Empirical Results:}
We empirically compare PBT and {ITR} in a 2-D navigation environment with 1 agent and $N_L$ landmarks in Fig.~\ref{fig:pbt-vs-iter}.
The reward is 1 if the agent successfully navigates to a landmark and 0 otherwise.
We train $N_L$ policies using both PBT and {ITR} to discover strategies toward each of these landmarks.
More details can be found in App.~\ref{app:imple}.
Table~\ref{tab:pbt-vs-iter} shows the number of discovered landmarks by PBT and {ITR}.
{ITR} performs consistently better than PBT even in this simple example.
We intuitively illustrate the learning process of PBT and {ITR} in Fig.~\ref{fig:pbt-vs-iter}. {ITR}, due to its computation efficiency, can afford to run longer iterations and tolerate larger exploration noises. Hence, it can converge easily to diverse solutions by imposing a large diversity constraint. PBT, however, only converges when the exploration is faint, otherwise it diverges or converges too slowly.

\subsection{Practical Remark}

Based on the above analyses, we suggest ITR and diversity measures based on state distances be \emph{preferred} in RL applications.
We also acknowledge that, by the no-free-lunch theorem, they cannot be universal solutions and that trade-offs may still exist (see discussions in Sec.\ref{sec:conclusion} and App.\ref{app:discussion}).
Nonetheless, in the following sections, we will show that the effective implementation of these choices can lead to superior performances in various challenging benchmarks.
We hope that our approach will serve as a starting point and provide valuable insights into the development of increasingly powerful algorithms for potentially more challenging scenarios.

%% file: 5-method.tex
\section{Method}
\label{sec:method}

In this section, we develop a diversity-driven RL algorithm, \emph{State-based Intrinsic-reward Policy Optimization ({\name})}, by combining {ITR} and state-distance-based measures.
{\name} runs $M$ iterations to discover $M$ distinct policies.
At the $i$-th iteration, we solve equation~(\ref{eq:iter-co}) by converting it into unconstrained optimization using the Lagrange method. The unconstrained optimization can be written as:

\vspace{-7mm}
\begin{footnotesize}
\begin{align}
\min_{\pi_i}\,\max_{\lambda_j\ge 0,\,1\le j < i}\,  -J(\pi_i) - \sum_{j=1}^{i-1}
\lambda_j \left(D_\gS(\pi_i,\pi^\star_j) - \delta\right)
\label{eq:uco}
\end{align}
\end{footnotesize}
\vspace{-3mm}

$\lambda_j$ ($1\le j<i$) are Lagrange multipliers.
% and bounded in the interval $[0,\Lambda]$ for a large parameter $\Lambda$. 
$\{\pi^\star_j\}_{j=1}^{i-1}$ are previously obtained policies.
We adopt two-timescale Gradient Descent Ascent (GDA)~\citep{lin2020gradient} to solve the above minimax optimization, i.e., performing gradient descent over $\pi_i$ and gradient ascent over $\lambda_j$ with different learning rates. In our algorithm, we additionally enforce the dual variables $\lambda_j$ to be bounded (i.e., in an interval $[0,\Lambda]$ for a large number $\Lambda$), which plays an important role both in the theoretical analysis and in empirical convergence.
However, $D_\gS(\pi_i,\pi_j^\star)$ cannot be directly optimized w.r.t. $\pi_i$ through gradient-based methods because it depends on the states traversed by $\pi$, rather than its output (e.g. actions).
Therefore, we cast $D_\gS(\pi_i,\pi_j^\star)$ as the cumulative sum of intrinsic rewards, specifically the intrinsic return. This allows us to leverage policy gradient techniques for optimization.
The pseudocode of {\name} can be found in App.~\ref{app:algorithmic}.

An important property of {\name} is the convergence guarantee.
We present an informal illustration in Thm.~\ref{thm:converge} and present the formal theorem with proof in App.~\ref{app:proof2}.
% \subsection{Convergence Guarantee}
\begin{theorem}
    (Informal) Under continuity assumptions, {\name} converges to an $\epsilon$-stationary point.\label{thm:converge}
\end{theorem}

\textbf{Remark:}
% Please {refer to} the appendix for detailed assumptions and proof.
We assumed that the {return} $J$ and the distance $D_\gS$ are smooth in policies. In practice, this is true if (1) policy and state space are bounded and (2) reward function and system dynamics are continuous in the policy. (Continuous functions are bounded over compact spaces.) %Without these assumptions, negative results could be proved.
The key step is to analyze the role of the bounded dual variables $\lambda$, which achieves an $\frac{1}{\Lambda}$-approximation of constraint without hurting the optimality condition.
% \revision{
% The assumptions in Thm.~\ref{thm:converge} may not hold in general RL environments.
% However, we find that even in environments that potentially violate these assumptions, SIPO can still perform well.
% We believe that the practical impact of this limitation is negligible.
% }

% We provide two practical realizations of $D_\gS$.
Instead of directly defining $D_\gS$, we define intrinsic rewards as illustrated in Sec.~\ref{sec:method}, such that $D_\gS(\pi_i,\pi_j^\star)=\E_{s_h\sim\mu_{\pi_i}}\left[\sum_{h=1}^H r_\textrm{int}(s_h;\pi_i,\pi^\star_j)\right]$.

\textbf{RBF Kernel:}
The most popular realization of Eq.~(\ref{eq:state-dist}) in machine learning is through kernel functions. Herein, we realize Eq.~(\ref{eq:state-dist}) as an RBF kernel on states.
Formally, the intrinsic reward is defined by
\begin{small}
\begin{align}
r^\textrm{RBF}_\textrm{int}(s_h;\pi_i,\pi^\star_j)=\frac{1}{H}\E_{s^\prime\sim \mu_{\pi^\star_j}}
\left[
-\exp\left(-\frac{\Vert s_h-s^\prime\Vert^2}{2\sigma^2}\right)
\right]
\end{align}
\end{small}
where $\sigma$ is a hyperparameter controlling the variance.

\textbf{Wasserstein Distance:}
For stronger discrimination power, we can also realize Eq.~(\ref{eq:state-dist}) as $L_2$-Wasserstein distance. According to the dual form~\citep{villani2009optimaltransport}, we define
\begin{small}
\begin{align}
r^\textrm{WD}_\textrm{int}(s_h;\pi_i,\pi_j^\star)=\frac{1}{H}\sup_{\Vert f\Vert_L\le 1}
 f(s_h)
-\E_{s^\prime\sim \mu_{\pi^\star_j}}\left[ f(s^\prime) \right]\label{eq:ir-wd}
\end{align}
\end{small}
where $f:\gS\to\R$ is a $1$-Lipschitz function.
This realization holds a distinct advantage due to its interpretation within optimal transport theory~\cite{villani2009optimaltransport,wgan}. Unlike distances that rely solely on specific summary statistics such as means, Wasserstein distance can effectively quantify shifts in state distributions and remains robust in the presence of outliers~\cite{villani2009optimaltransport}.
We implement $f$ as a neural network and clip parameters to $[-0.01,0.01]$ to ensure the Lipschitz constraint. {Note that $r^\textrm{WD}_\textrm{int}$ incorporates representation learning by utilizing} a learnable scoring function $f$ and is more flexible in practice. {We also show in App.~\ref{app:additional-ablation} that $r^\textrm{WD}_\textrm{int}$ is robust to different inputs, including states with random noises and RGB images.}

We name {\name} with $r^\textrm{RBF}_\textrm{int}$ and $r^\textrm{WD}_\textrm{int}$ \textbf{\emph{{\namerbf}}} and \textbf{\emph{{\namewd}}} respectively.

\textbf{Implementation:} To incorporate temporal information, we stack the recent 4 global states to compute intrinsic rewards and normalize the intrinsic rewards to stabilize training. 
In multi-agent environments, we learn an agent-ID-conditioned policy~\citep{fu2022revisiting} and share the parameter across all agents.
Our implementation is based on MAPPO~\citep{yu2021surprising} with more details in App.~\ref{app:imple}.

%% file: 6-experiments.tex
\section{Experiments}
\label{sec:exp}

We evaluate {\name} across three domains that exhibit multi-modality of solutions. The first domain is the humanoid locomotion task in Isaac Gym~\citep{makoviychuk2021isaac}, where diversity can be quantitatively assessed by well-defined behavior descriptors.
{We remark that the issues we addressed in Sec.~\ref{sec:case:measure} may not be present in this task where the action space is small and actions are highly correlated with states.
Further, we examine the effectiveness of {\name} in two much more challenging multi-agent domains}, StarCarft Multi-Agent Challenge (SMAC)~\citep{samvelyan2019starcraft} and Google Research Football (GRF)~\citep{kurach2020google}, where well-defined behavior descriptors are not available and existing diversity measures may produce misleading diversity scores. We provide introductions to these environments in App.~\ref{app:env}.

First, we show that {\name} can efficiently learn diverse strategies and outperform several baseline methods, including DIPG~\cite{masood2019diversity}, SMERL~\citep{kumar_one_2020}, DvD~\cite{parker-holder_effective_2020},
% TrajDi~\citep{lupu_trajectory_2021},
and RSPO~\cite{zhou_continuously_2022}.
Then, we qualitatively demonstrate the emergent behaviors learned by {\name}, which are both \emph{visually distinguishable} and \emph{human-interpretable}.
Finally, we perform an ablation study over the building components of {\name} and show that both the diversity measure, ITR, and GDA are critical to the performance.

All algorithms run for the same number of environment frames on a desktop machine with an RTX3090 GPU. Numbers are average values over 5 seeds in Humanoid and SMAC and 3 seeds in GRF with standard deviation shown in brackets.
More algorithm details can be found in App.~\ref{app:imple}.
Additional visualization results can be found on our project website (see App.~\ref{app:web}).
%. All experiments are conducted on a desktop with a single NVIDIA RTX3090 GPU.
% Experiment details can be found in App.~\ref{app:imple}.

% \vspace{-4mm}
\subsection{Comparison with Baseline Methods}

\begin{wraptable}{r}{5.0cm}
\scriptsize
\centering
\vspace{-4mm}
\caption{\small{Pairwise distance of joint torques (i.e., diversity scores) in the humanoid locomotion task.}}
\label{tab:humanoid-diversity-score}
\begin{tabular}{cccccc}
\toprule  
 {\namerbf} & {\namewd} & RSPO \\
\midrule
0.53(0.17) & \textbf{0.71(0.23)} & 0.53(0.05) \\ 
\bottomrule
\toprule  
DIPG & DvD & SMERL \\
\midrule
0.12(0.04) & 0.40(0.22)& 0.01(0.00) \\
\bottomrule
\end{tabular}
\vspace{-2mm}
\end{wraptable}
\textbf{Humanoid Locomotion.}
Following \citet{zhou_continuously_2022}, we train a population of size $4$. We assess diversity by the pairwise distance of joint torques, a widely used behavior descriptor in recent Quality-Diversity works~\cite{wuquality}. Torque states are not included as the input of diversity measures and we only use them for evaluation to ensure a fair comparison. Results are shown in Table~\ref{tab:humanoid-diversity-score}. We can see that both variants of SIPO can outperform all baseline methods except that SIPO-RBF achieves comparable performance with RSPO, even if RSPO explicitly encourages the output of different actions/forces.

\begin{wraptable}{r}{5.6cm}
\scriptsize
\centering
\vspace{-4mm}
\caption{\small{State entropy estimated by $k$-nearest-neighbor in SMAC. ($k=12$)}}
\label{tab:state-entropy-smac}
\begin{tabular}{ccc}
\toprule
           &         \emph{2m\_vs\_1z} &       \emph{2c\_vs\_64zg} \\
\midrule
\namerbf & \textbf{0.038(0.002)} & \textbf{0.072(0.003)} \\
   \namewd & 0.036(0.001) & 0.056(0.003) \\
      RSPO & 0.032(0.003) & 0.070(0.001) \\
      DIPG & 0.032(0.002) & 0.056(0.004) \\
     SMERL & 0.028(0.002) & 0.042(0.002) \\
       DvD & 0.030(0.002) & 0.057(0.003) \\
\bottomrule
\end{tabular}
% \vspace{-5mm}
\end{wraptable}
\textbf{SMAC}
Following \citet{zhou_continuously_2022}, we run {\name} and all baselines on an easy map, \emph{2m\_vs\_1z}, and a hard map, \emph{2c\_vs\_64zg}, both across 4 iterations.
We merge all trajectories produced by the policy collection and incorporate a $k$-nearest-neighbor state entropy estimation \cite{singh2003nearest} to assess diversity.
Intuitively, a more diverse population should have a larger state entropy value.
We set $k=12$ following \citet{liu2021behavior} and show results in Table~\ref{tab:state-entropy-smac}.
On these maps, two agents are both involved in the attack.
Therefore, RSPO, which incorporates an action-based cross-entropy measure, can perform well across all baselines.
However, {\name} explicitly compares the distance between resulting trajectories and can even outperform RSPO, leading to the most diverse population.

\textbf{GRF} We consider three academy scenarios, specifically \emph{3v1}, \emph{counterattack} (\emph{CA}), and \emph{corner}.
The GRF environment is more challenging than SMAC due to the large action space, more agents, and the existence of duplicate actions.
We determine a population size $M=4$ by balancing resources and wall-clock time across different baselines.
Table~\ref{tab:fb-baseline} compares the number of distinct policies (in terms of ball-passing routes, see App.~\ref{app:evaluation-protocol}) discovered in the population.
Due to the strong adversarial power of our diversity measures and the application of GDA, SIPO is the most efficient and robust --- even in the challenging 11-vs-11 \emph{corner} and \emph{CA} scenario, {\name} can effectively discover different winning strategies in just a few iterations across different seeds.
By contrast, baselines suffer from learning instability in these challenging environments and tend to discover policies with slight distinctions.
We also calculate the estimated state entropy as we did in SMAC. However, we find that this metric cannot distinguish fine-grained ball-passing behaviors in GRF (check our discussions in App.~\ref{app:addtional}).

\textbf{Remark:}
In GRF experiments, when $M$ is small, even repeated training with different random seeds (PG) is a strong baseline (see Table~\ref{tab:fb-baseline}). Hence, the numbers are actually restricted in a small interval (with a lower bound equal to PG results and an upper bound equal to $M=4$), which makes the improvements by SIPO seemingly less significant. However, achieving clear improvements in these challenging applications remains particularly non-trivial.
With a population size $M=10$, SIPO clearly outperforms baselines by consistently discovering one or more additional strategies.

\begin{table*}[b]
\vspace{-4mm}
\centering
\begin{threeparttable}
\centering
\caption{Number of distinct strategies in GRF discovered by different methods in terms of the ball-passing route.
Details of the evaluation protocol can be found in App.~\ref{app:evaluation-protocol}.}
\scriptsize
\begin{tabular}{ccccccccccc}
\toprule
 &\multirow{2}{*}{Population Size $M$} &  \multicolumn{2}{c}{ours} & \multicolumn{4}{c}{baselines}& {random}\\
 \cmidrule(lr){3-4}\cmidrule(lr){5-8}\cmidrule(lr){9-9}
 & & {\namerbf} & {\namewd} & DIPG & SMERL & DvD\tnote{1} & RSPO & {PG}\\
 \midrule
\emph{3v1}    &4&  \textbf{3.0 (0.8)}  & \textbf{3.0 (0.0)} & 2.7 (0.5) & 1.3 (0.5)  & \textbf{3.0 (0.8)} & 2.0 (0.0) & {2.7 (0.5)}\\
\emph{CA}       &4& \textbf{3.3 (0.5)} & 3.0 (0.8) & 2.3 (0.5) & 1.3 (0.5) & - & 2.0 (0.0) & {1.7 (0.5)}\\
\emph{corner}   &4& 2.7 (0.5) & \textbf{3.0 (0.8)} & 1.7 (0.5) & 1.0 (0.0) & - & 1.6 (0.5) & {2.0 (0.8)}\\
\emph{3v1}    &10&  4.3 (0.5) & \textbf{5.7 (0.5)} & 3.7 (0.5) & - & - & 2.3 (0.5) & - \\
\bottomrule
\end{tabular}
\label{tab:fb-baseline}
\begin{tablenotes}\scriptsize
\item[1] Training DvD in \emph{CA} and \emph{corner} or with $M=10$ requires $>$24GB GPU memory, which exceeds our memory limit.
\end{tablenotes}
\end{threeparttable}
\vspace{-3mm}
\end{table*}

\subsection{Qualitative Analysis}
\label{sec:main-result}

\begin{figure*}[t]
    \centering
    % \vspace{-5mm}
    \includegraphics[width=0.95\textwidth]{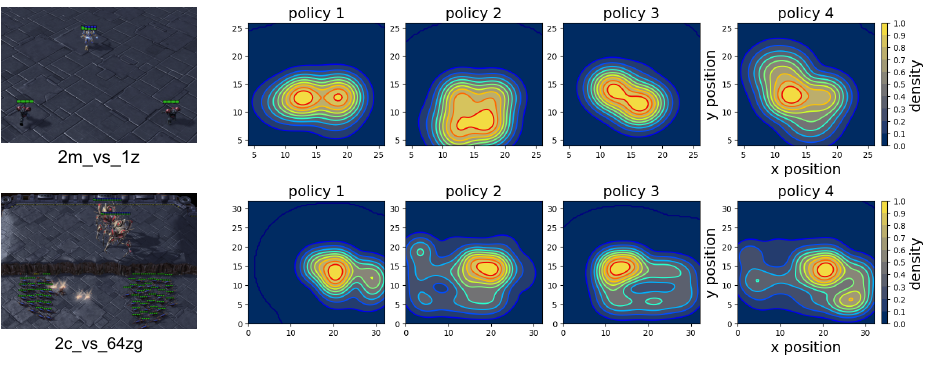}
    \caption{Heatmaps of agent positions in SMAC across 4 iterations with {\namerbf}.}
    \label{fig:smac-heatmap}
\end{figure*}

For SMAC, we present heatmaps of agent positions in Fig.~\ref{fig:smac-heatmap}.
The heatmaps clearly show that {\name} can consistently learn novel winning strategies to conquer the enemy.
Fig.~\ref{fig:3v1} presents the learned behavior by {\name} in the GRF \emph{3v1} scenario of seed 1.
We can observe that agents have learned a wide spectrum of collaboration strategies across merely 7 iterations.
The strategies discovered by {\name} are both \emph{diverse} and \emph{human-interpretable}.
% We take the \emph{3v1} scenario as an example.
In the first iteration, all agents are involved in the attack such that they can distract the defender and obtain a high win rate.
The 2nd and the 6th iterations demonstrate an efficient pass-and-shoot strategy, where agents quickly elude the defender and score a goal.
In the 3rd and the 7th iterations, agents learn smart ``one-two'' strategies to bypass the defender, a prevalent tactic employed by human football players. {We note that \emph{NONE} of the baselines have ever discovered this strategy across all runs, while SIPO is consistently able to derive such strategies for all random seeds.}
Visualization results in \emph{CA} and \emph{corner} scenarios can be found in App.~\ref{app:addtional}.

\begin{figure*}[t]
    \centering
    \includegraphics[width=0.95\textwidth]{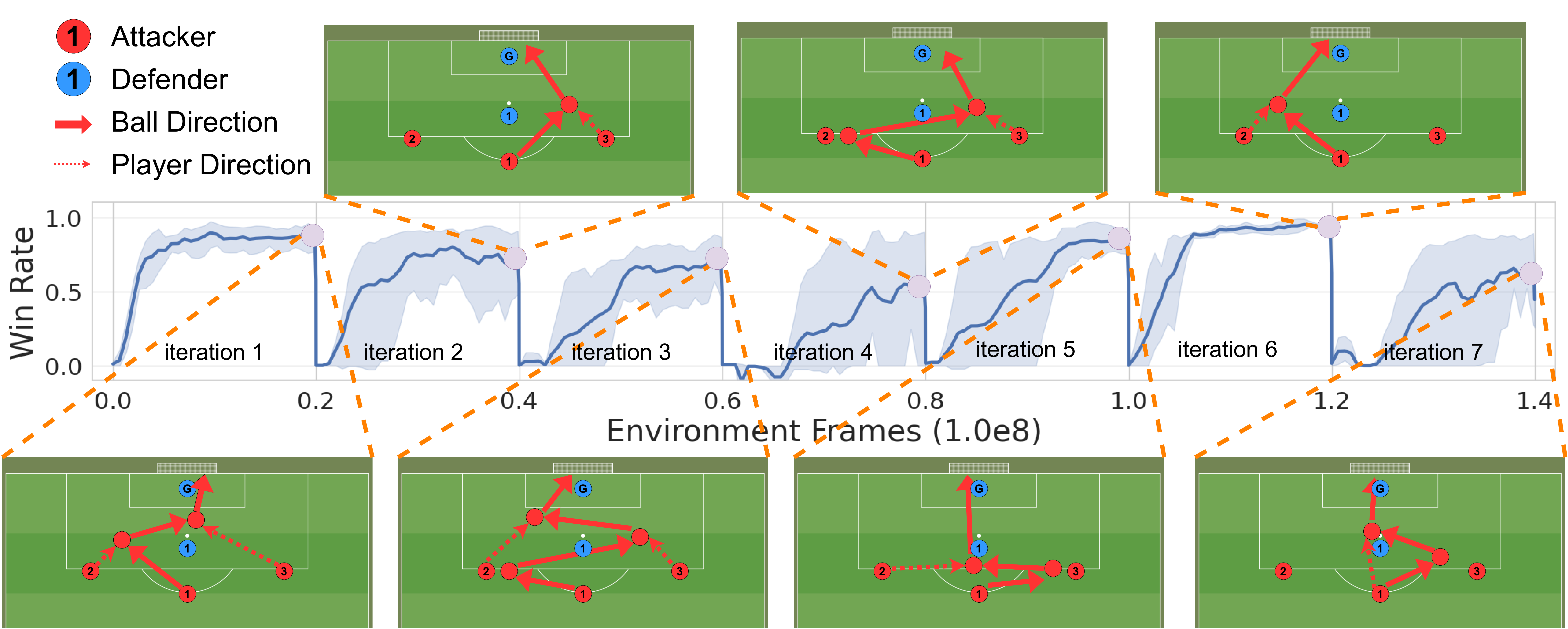}
    \caption{Learning curves and discovered strategies
    by {\namewd} in the \emph{3v1} scenario over 7 iterations.
    Strategies of seed 1 are shown.}
    \label{fig:3v1}
\end{figure*}

\subsection{Ablation Study}

We apply these changes to {\namewd}:
\vspace{-1.5mm}
\begin{footnotesize}
\begin{itemize}
    \item \textit{fix-L}: Fixing the multiplier $\lambda_i$ instead of applying GDA.
    \item \textit{CE}: The intrinsic reward is replaced with cross-entropy, i.e., $r_{\textrm{int}}^{\textrm{CE}}(s_h,a_h)=-\log \pi_j^\star (a_h \mid s_h)$, where $\pi_j^\star$ denotes a previously discovered policy. Additionally, GDA is still applied.
    \item \textit{filter}: Optimizing the extrinsic rewards on trajectories that have intrinsic returns exceeding $\delta$ and optimizing intrinsic rewards defined by Eq.~(\ref{eq:ir-wd}) for other trajectories~\cite{zhou_continuously_2022}.
    \item \textit{PBT}: Simultaneously training $M$ policies with $M(M-1)/2$ constraints (i.e., directly solving Eq.~(\ref{eq:pbt-co})) with intrinsic rewards defined by Eq.~(\ref{eq:ir-wd}) and GDA.
\end{itemize}
\end{footnotesize}
\vspace{-1.5mm}

\begin{wraptable}{r}{8.5cm}
\centering
\begin{threeparttable}
\scriptsize
\centering
\caption{\small{\# distinct strategies of ablations in GRF.}}
% \vspace{-2mm}
\begin{tabular}{ccccccccccc}
\toprule
 % &  \multicolumn{2}{c}{ours} & \multicolumn{3}{c}{ablations}\\
 % \cmidrule(lr){2-3}\cmidrule(lr){4-6}
  & ours & fix-L & CE & filter & {PBT} \\
 \midrule
\emph{3v1}      & \textbf{3.0 (0.0)} & 1.0 (0.0) &2.7 (0.5) & 1.3 (0.5) & {2.7 (0.5)} \\
\emph{CA}       & \textbf{3.0 (0.8)} & -\tnote{1} &  2.3 (0.8) & 1.0 (0.0) & -\tnote{2}\\
\emph{corner}    & \textbf{3.0 (0.8)} & -\tnote{1} & 1.7 (0.5) & 1.0 (0.0) & -\tnote{2}\\
\bottomrule
\label{tab:fb-ablation}
\end{tabular}
\vspace{-3mm}
\begin{tablenotes}
    \item[1] \scriptsize{Not converged}.
    \item[2] \scriptsize{Training requires $>$24GB memory and exceeds our memory limit}.
\end{tablenotes}
\end{threeparttable}
% \vspace{-3mm}
\end{wraptable}
We report the number of visually distinct policies discovered by these methods in Table~\ref{tab:fb-ablation}.
Comparison between {\name} and CE demonstrates that the action-based cross-entropy measure may suffer from duplicate actions in GRF and produce nearly identical behavior by overly exploiting duplicate actions, especially in the \emph{CA} and \emph{corner} scenarios with 11 agents.
Besides, the fixed Lagrange coefficient, the filtering-based method, {and PBT} are all detrimental to our algorithm.
These methods also suffer from significant training instability.
Overall, the state-distance-based diversity measure, ITR, and GDA are all critical to the performance of {\name}.

%% file: A2-additional.tex
\section{Additional Results}
\label{app:addtional}

\subsection{More Qualitative Results}

\begin{figure}[hb]
    \centering
    \includegraphics[width=\textwidth]{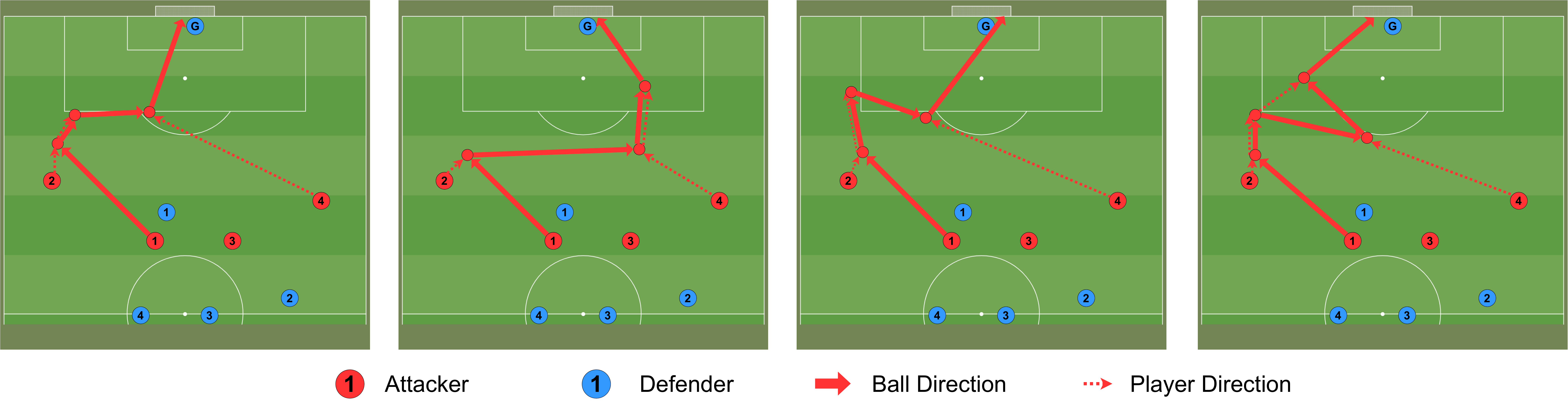}
    \caption{Visualization of learned behaviors in GRF \emph{CA} across a single training trial.}
    \label{fig:ca}
\end{figure}

\begin{figure}[hb]
    \centering
    \includegraphics[width=0.8\textwidth]{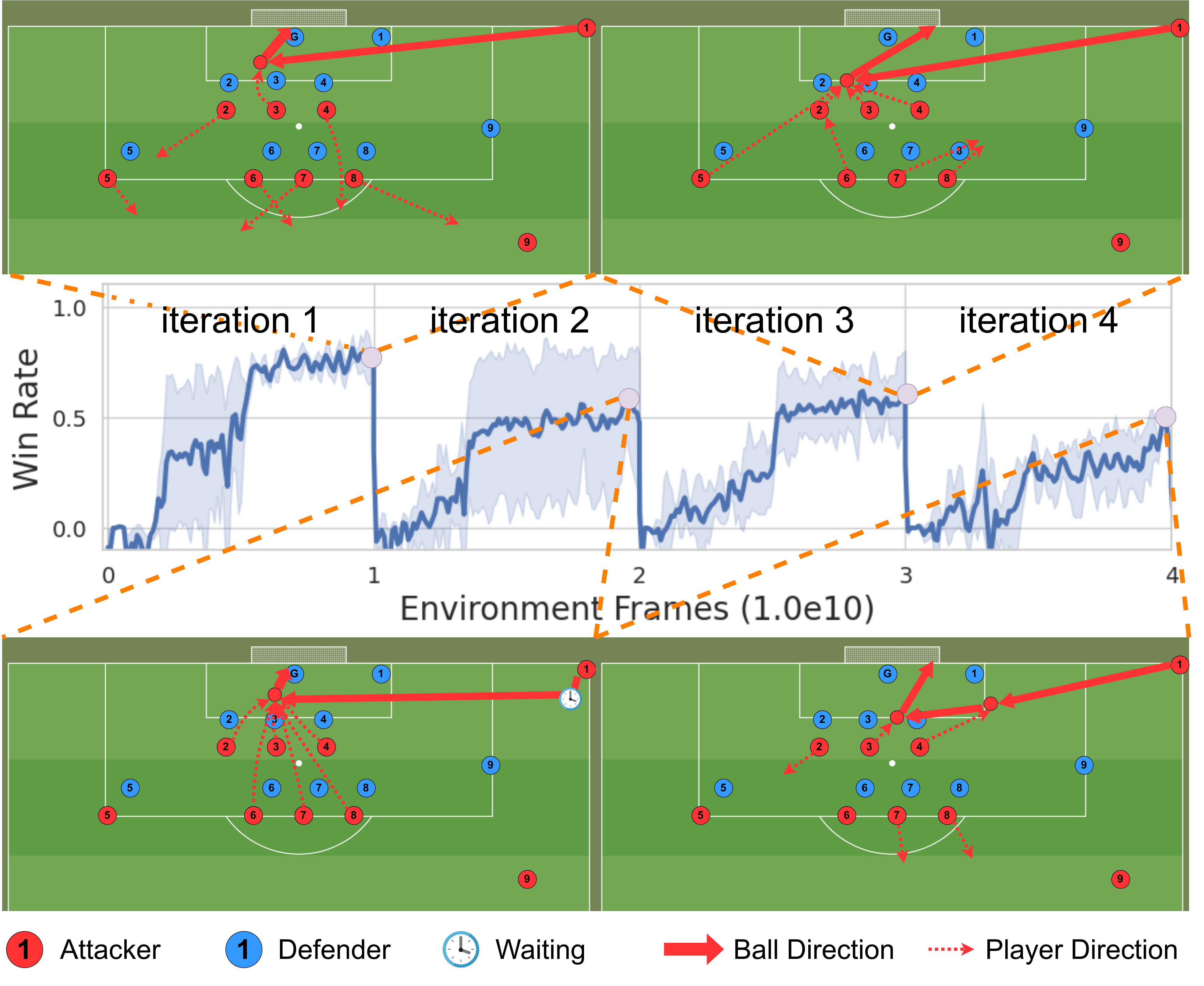}
    \caption{Visualization of learned behaviors in GRF \emph{corner}.}
    \label{fig:corner}
\end{figure}

\begin{figure}[hb]
    \centering
    \includegraphics[width=\textwidth]{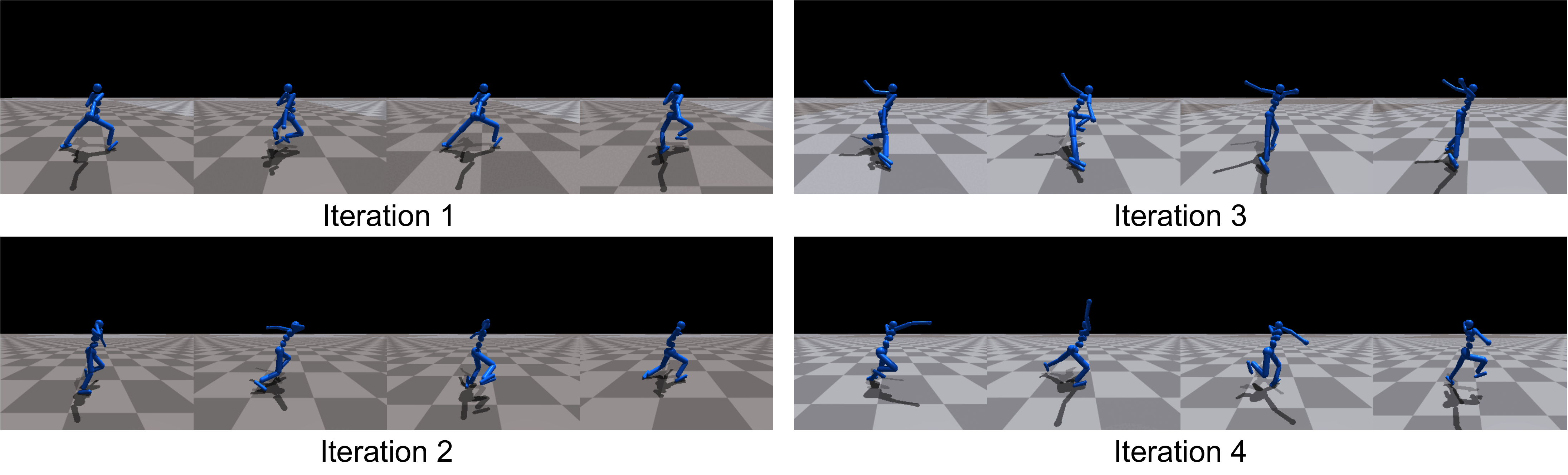}
    \caption{Visualization of learned behaviors in Humanoid.}
    \label{fig:state-humanoid}
\end{figure}

We show additional visualization results in Fig.~\ref{fig:ca}, Fig.~\ref{fig:corner}, and Fig.~\ref{fig:state-humanoid}.
Corresponding GIF visualizations can be found on our project website.

\begin{table}[ht]
\centering
\begin{threeparttable}
\centering
\caption{$k$-nearest neighbor state entropy estimation in GRF. Population size $M=4$.}
\label{tab:state-dvd-fb}
\scriptsize
\begin{tabular}{cccccccc}
\toprule
 &  \multicolumn{2}{c}{ours} & \multicolumn{4}{c}{baselines}\\
 \cmidrule(lr){2-3}\cmidrule(lr){4-7}
 & {\namerbf} & {\namewd} & DIPG & SMERL\tnote{1} & DvD\tnote{2} & RSPO\tnote{1} & PG (random seeding)\\
\midrule
   \emph{3v1} & 0.009(0.000) & 0.012(0.000) & 0.010(0.001) & 0.011(0.002) & 0.010(0.000) & 0.011(0.001) & 0.009(0.001) \\
    \emph{CA} & 0.037(0.000) & 0.031(0.006) & 0.036(0.002) &            - &            - & 0.034(0.001) & 0.039(0.001)\\
\emph{Corner} & 0.028(0.001) & 0.031(0.001) & 0.030(0.002) &            - &            - &            - & 0.028(0.002) \\
\bottomrule
\end{tabular}
\begin{tablenotes}\footnotesize
\item[1] The learned policy in some iterations cannot even collect a single winning trajectory, so we are unable to compute their diversity score.
\item[2] Training DvD in \emph{CA} and \emph{corner} requires $>$24GB GPU memory, which exceeds our memory limit.
\end{tablenotes}
\end{threeparttable}
\end{table}

\subsection{Task Performance Evaluation}

The evaluation win rates of the demonstrated visualization results in SMAC and GRF are shown in Table~\ref{tab:eval-win-rate}.
Evaluated episode returns in Humanoid are shown in Table~\ref{tab:eval-humanoid-ret}.

\cmrd{We also present the diversity score and average rewards achieved by baselines in Table~\ref{tab:baseline-reward}. These numerical values are averaged across the entire population for a clear comparison. The tabulated data highlights the varying trade-offs between task performance and diversity exhibited by different algorithms. It is noteworthy that SIPO, in particular, displays an adeptness at training a notably more diverse population while upholding a reasonably moderate level of task performance.}

\begin{table}[ht]
    \centering
    % \scriptsize
    \caption{Evaluation win rate (\%) of the demonstrated visualization results in SMAC and GRF.}
    \label{tab:eval-win-rate}
    \begin{tabular}{c cc ccc}\toprule
         & \multicolumn{2}{c}{SMAC} & \multicolumn{3}{c}{GRF}\\
         \cmidrule(lr){2-3}\cmidrule(lr){4-6}
         & \emph{2m1z}& \emph{2c64zg}   &\emph{3v1} & \emph{CA} & \emph{corner} \\
         \midrule
         $\pi_1$ &   100.0(0.0)     &   98.1(2.1)    &   92.3(6.2)    &   48.2(10.4)    &     78.2(16.2)     \\
         $\pi_2$ &   99.6(0.9)      &  100.0(0.0)    &   82.1(8.4)    &   43.8(42.2)    &     57.0(37.7)      \\
         $\pi_3$ &   100.0(0.0)     &   96.9(3.3)    &   90.7(1.1)    &   54.7(30.6)   &     55.7(20.8)     \\
         $\pi_4$ &   99.6(0.6)      &   98.6(2.4)    &   63.6(45.0)    &   17.2(30.0)    &    30.7(29.0)        \\
         $\pi_5$ &   -      &   -    &  85.4(9.1)     &   -   &     -       \\
         $\pi_6$ &   -      &   -    &  93.2(1.9)     &   -   &     -       \\
         $\pi_7$ &   -      &   -    &  64.6(32.5)     &   -   &     -       \\
         \bottomrule
    \end{tabular}
\end{table}

\begin{table}[ht]
    \centering
    \caption{Episode returns in Humanoid.}
    \label{tab:eval-humanoid-ret}
    \begin{tabular}{c ccc}\toprule
         & SIPO-RBF & SIPO-WD & SIPO-WD (visual) \\
         \midrule
         $\pi_1$ &   4863.9(970.3)     &   3909.4(533.4)    &  4761.3(107.8)  \\
         $\pi_2$ &   3746.5(488.0)     &  3784.2(481.2)     &   4349.3(169.0)     \\
         $\pi_3$ &   3092.0(805.0)     &   3770.4(674.4)    &   4724.3(946.5)    \\
         $\pi_4$ &   2332.8(519.8)      &   3589.6(387.4)    &   3819.7(588.7)    \\
         \bottomrule
    \end{tabular}
\end{table}

\begin{table}[ht]
    \centering
    \caption{Reward/diversity of all baselines. The reward metric in SMAC and GRF are evaluation win rate (\%). The evaluation metrics of diversity used in humanoid, SMAC, GRF are the joint torque distance, state entropy (1e-3), and the number of different ball-passing routes, respectively.
    It is noteworthy that SIPO, in particular, displays an adeptness at training a notably more diverse population while upholding a reasonably moderate level of task performance.}
    \scriptsize
    \begin{tabular}{c|ccccccc}
    \toprule
        Task/Scenario & SIPO-RBF & SIPO-WD & DIPG & RSPO & SMERL & DvD & PPO \\
        \midrule
        humanoid 	&	3508 / 0.53 & 	3763 / 0.71 	& 5191 / 0.12	 & 1455 / 0.53   &4253 / 0.01 & 4498 / 0.40 & 5299 / - \\
        SMAC 2m1z  & 100 / 38 	& 	100 / 36	& 100 / 32	& 100 / 32	& 100 / 28 & 100 / 30 & 100 / - \\
        SMAC 2c64zg & 	99 / 72	& 	93 / 56 	& 99 / 70	& 85 / 56	& 100 / 42 & 100 / 57 & 100 / - \\
        GRF 3v1 (first 4) & 	93 / 3.0	& 	 82 / 3.0	& 93 / 2.7	&  94 / 2.0	&   91 / 1.3    &   83 / 3.0  &    92 / 2.7  \\
        GRF CA 	& 	70 / 3.3	& 	41 / 3.0	& 46 / 2.3	&  76 / 2.0	&    45 / 1.3    &   -    &   50 / 1.7   \\
        GRF Corner 	& 	72 / 2.7	& 	56 / 3.0	& 75 / 1.7	&  23 / 1.6	&    67 / 1.0    &    -   &   71 / 2.0    \\
        \bottomrule
    \end{tabular}
    \label{tab:baseline-reward}
\end{table}

\subsection{Evaluation Metric and Protocol for Diversity}
\label{app:evaluation-protocol}

\subsubsection{Humanoid}

The Humanoid locomotion task is well-studied in the Quality-Diversity (QD) community, enabling the application of well-defined behavior descriptors (BD) to assess diversity scores. While domain-agnostic metrics like DvD scores can also be applied, we consider domain-specific BDs to be more appropriate and accurate for evaluation in this setting.

\subsubsection{SMAC}

Complex multi-agent tasks like SMAC lack well-defined BDs. Hence, domain-agnostic diversity measures such as the state-entropy measure should be applied. Moreover, different SMAC winning strategies tend to visit different areas of the map, which can be usually captured by the state-entropy measure.

\subsubsection{GRF}
In our initial study of the GRF task, diversity was evaluated using the $k$-nearest-neighbor state entropy estimation as in SMAC (see Table~\ref{tab:state-dvd-fb}). However, we observed a significant difference between the computed scores and visualized behaviors. Further investigation revealed that state entropy can sometimes report fake diversity in GRF.
For example, the ball-moving route is highly fine-grained between nearby players in the counter-attack (CA) scenario, and additional passes may not change the state entropy significantly. Instead, agents' positions play a crucial role in this scenario, where different shooting positions can introduce substantial state variance and lead to a higher entropy score.
As an example, readers can refer to the replays of \href{https://mega.nz/folder/4Sth3AyJ#m41lG_Yfya9h30KWR98A0A}{{\namerbf} (4 iterations of seed 2)} and \href{https://mega.nz/folder/gbl1yCQC#ElQCgVdi-Xwnb1X-cTtKOA}{PG (seed 2, 1002, 2002, and 3002)}, where SIPO-RBF discovers four distinct passing strategies, while PG keeps passing the ball to the same player. Nevertheless, the state entropy of PG (0.0397) is higher than that of SIPO-RBF (0.0378).

Hence, we counted the number of distinct policies according to their ball-passing routes, such as passing the ball to different players or shooting with different players, to evaluate diversity in GRF. To quantify these differences, we extracted the positions of the ball and the players in the field and calculated the nearest ally player ID to the ball across a winning episode. We then removed timesteps where the nearest distance was above a pre-defined threshold of $0.03$. Typically, these timesteps correspond to instances when the ball is being transferred among players, making the nearest player ID irrelevant. Next, we removed consecutive duplicate player IDs from the resulting sequence to obtain a concise and informative embedding of the ball passing route. By comparing the lengths of their respective embeddings and verifying that the player IDs in each embedding are identical, we determined whether the two policies exhibit similar behavior.

We acknowledge that existing diversity measures may not be applicable in GRF, and hence we opted for this novel approach to evaluate diversity. Additionally, we experimented with using raw observations, which include ball ownership information provided by the game engine, but found it to be highly inaccurate based on our visualization.

\subsection{{Additional Ablation Studies}}
\label{app:additional-ablation}

\subsubsection{Input to the Diversity Measure}
\label{app:ablation-state-input}

\paragraph{Vectorized States in Google Research Football}
% To show that {\name} has the ability to discover human-interpretably diverse strategies,
We perform an additional ablation study over the input of our diversity measure in GRF \emph{3v1} scenario with {\namewd}. We consider the following kinds of state input besides the default state input we adopted in Sec.~\ref{sec:exp}:
\begin{itemize}
    % \small
    \item full observation (named \emph{full}, $115$ dims);
    \item default state input with random noises of the same dimension (named \emph{random}, $36$ dims).
\end{itemize}
The numbers of visually distinct strategies are listed in Table~\ref{tab: state subsets}. The performance of \emph{full} and \emph{random} is similarly good. The result implies that the learnable discriminator can automatically filter out irrelevant states to some extent, and that {\namewd} performs relatively robust w.r.t. different state input of the diversity measure.
\begin{table}[ht]
    \centering
    % \scriptsize
    \caption{State input ablation. The table shows the number of distinct strategies in GRF \emph{3v1}.}
    \label{tab: state subsets}
    \begin{tabular}{cccc}
    \toprule
         & {\namewd} & full &random \\
         \midrule
         \emph{3v1} &   3.0 (0.0) & 3.0 (0.8)       &   3.0 (0.0)    \\
         \bottomrule
    \end{tabular}
\end{table}

\paragraph{RGB Images in Locomotion Tasks}
We run {\namewd} in the visual Humanoid task based on Isaac Gym~\citep{makoviychuk2021isaac}. The training protocol is similar to the state-only version (i.e., the input of policy and intrinsic rewards are both locomotion states of the Humanoid) except that we stack recent 4 RGB camera observations ($84\times84$) as the input of intrinsic rewards in Eq.~\ref{eq:ir-wd}. We adopt the training code developed in Isaac Gym and the default PPO configuration. The backbone of the discriminator is composed of 4 convolutional layers with kernel size 3, stride 2, padding 1, and [16, 32, 32, 32] channels. Then the feature is passed to an MLP with 1 hidden layer and 256 hidden units. The activation function is leaky ReLU with slope $0.2$.
We also compute the pairwise distance of joint torques as in the state-only version and show the result in Table~\ref{tab:humanoid-visual}. Visualizations are shown in Fig.~\ref{fig:visual-humanoid}. {\namewd} can also learn meaningful diverse behaviors with RGB images as the state input thanks to the learnable Wasserstein discriminator. This implies that our algorithm can be naturally extended to high-dimensional states and incorporated with advances in representation learning, which may be a potential future direction.

\begin{figure}[b]
    \centering
    \includegraphics[width=\textwidth]{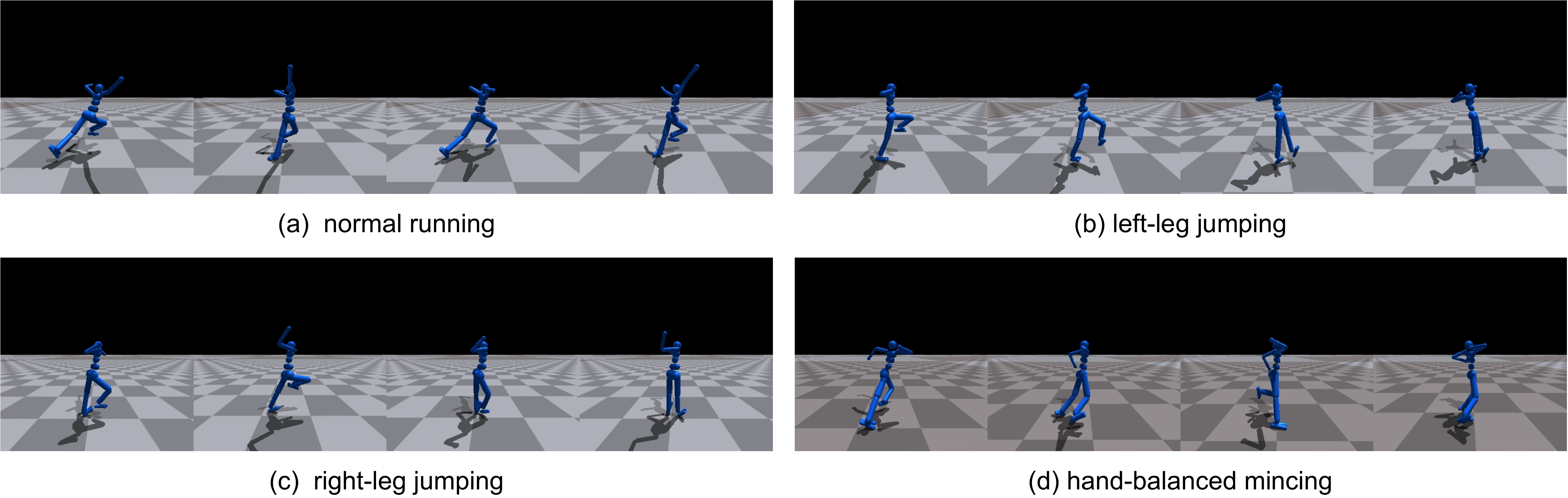}
    \caption{Results of {\namewd} in the visual Humanoid task.}
    \label{fig:visual-humanoid}
\end{figure}

\begin{table}[t]
    \centering
    % \scriptsize
    \caption{Pairwise distance of joint torques (i.e., diversity score) in Humanoid with visual input. Results in visual experiments are averaged over 3 seeds.}
    \label{tab:humanoid-visual}
    \begin{tabular}{c cc ccc}
    \toprule
         {\namewd} (visual) & {\namewd} & RSPO (best baseline) \\
         \midrule
         0.62 (0.26) &   0.71 (0.23) & 0.53 (0.05)  \\
         \bottomrule
    \end{tabular}
\end{table}

\subsubsection{Combining State- and Action-based Diversity Measures}

\cmrd{Based on SIPO-RBF, we introduce additional action information by directly concatenating the global state, used for diversity calculation, with the one-hot encoded actions of all agents within the GRF domain. Table~\ref{tab:state-action-combine} presents the outcomes, indicating the number of policies obtained. For scenarios with a limited number of agents, the action-augmented variant demonstrates comparable performance. However, when the agent count increases (as evident in the 11-agent cases of CA and corner), the incorporation of actions can introduce misleading diversity, detracting from the authenticity of the outcomes.}

\begin{table}[ht]
    \centering
    \begin{tabular}{c|ccc}
         &  3v1 & CA & Corner\\
         SIPO-RBF & 3.0 (0.8)	&3.3 (0.5)	&2.7 (0.5) \\
        SIPO-RBF w. Action & 	3.0 (0.0)	&2.3 (0.5)	&1.0 (0.0)\\
    \end{tabular}
    \caption{Ablation study of combining state- and -action-based diversity measures. The number of different strategies across a population of 4 is shown with standard deviation in the brackets.}
    \label{tab:state-action-combine}
\end{table}

\subsection{How to Adjust Constraint-Related Hyperparameters}

\begin{figure}
\begin{minipage}{0.34\linewidth}
\centering
\includegraphics[width=\textwidth]{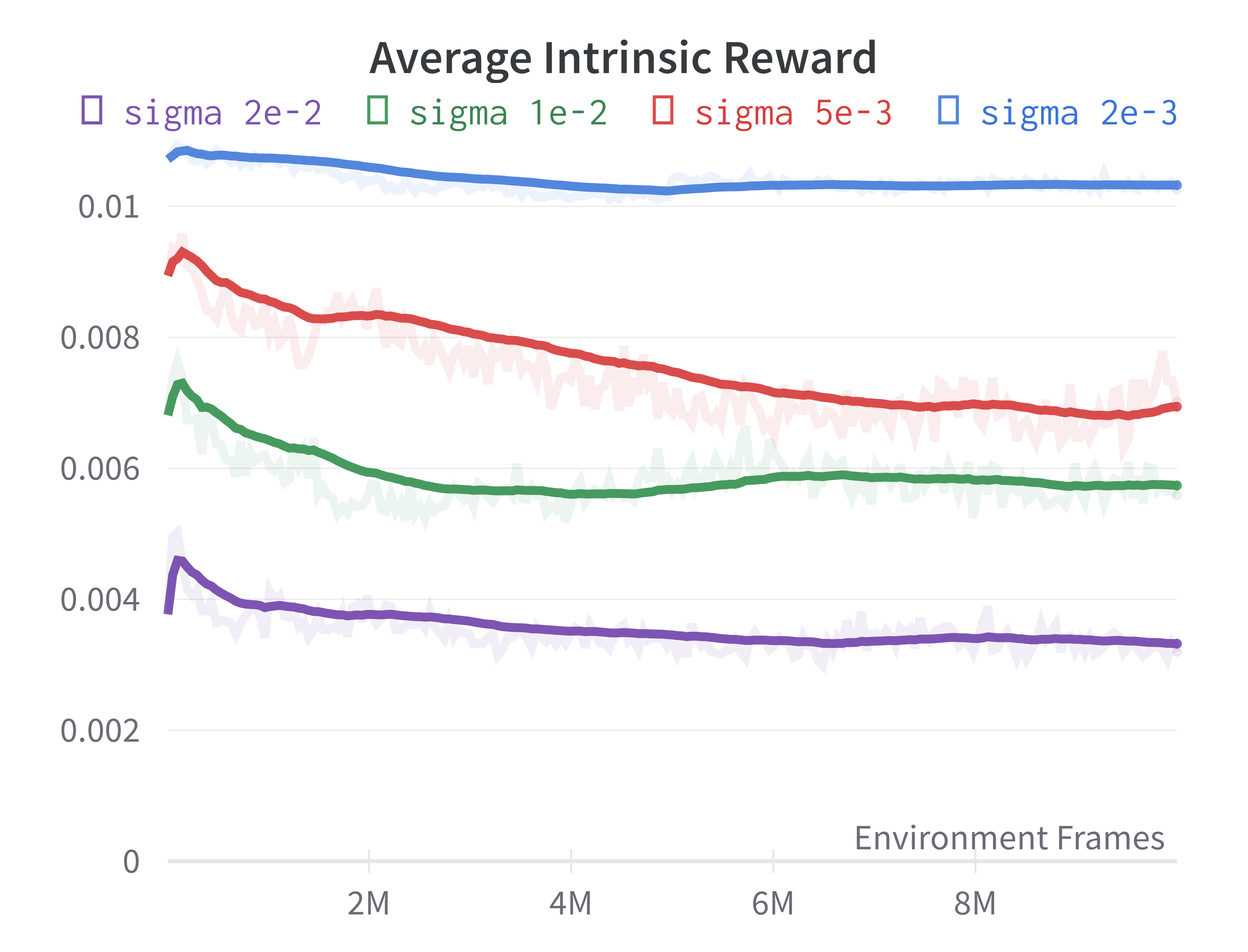}
\caption{Average intrinsic reward during training $\pi_1$.}
\label{fig:rint-trend}
\end{minipage}
\hfill
\begin{minipage}{0.65\linewidth}
\centering
\captionof{table}{The values of $\delta$ and $\alpha$ in different environments.}
    \label{tab:diversity hyperparameter}
    \begin{tabular}{ccccccc}
    \toprule  
     & \multicolumn{3}{c}{football}&\multicolumn{2}{c}{smac} \\
    \cmidrule(lr){2-4} \cmidrule(lr){5-6}
     & \emph{3v1} & \emph{corner} & \emph{CA} & 2m\_vs\_1z  & 2c\_vs\_64zg \\
    \midrule
    $\delta^\textrm{WD}$ & $0.004$ & $0.01$ & $0.012$ & $0.02$  & $0.2$\\
    $\alpha^\textrm{WD}$ & $1$ & $1$ & $0.5$ & $0.5$ & $0.05$ \\ 
    $\delta^\textrm{RBF}$ & 0.03 & 0.01 & 0.015 & 0.002 & 0.001\\
    $\alpha^\textrm{RBF}$ & 0.001 & 0.001 & 0.001 & 0.001 & 0.001\\
    $\sigma^2$ & 0.02 & 0.02 & 0.02 & 0.02 & 0.02 \\
    \bottomrule
    \end{tabular}
\end{minipage}
\end{figure}

Three hyperparameters are essential in the implementation of the intrinsic reward $r_\textrm{int}$: the threshold $\delta$, the intrinsic reward scale factor $\alpha$, and the variance factor $\sigma$ in $r_\textrm{int}^\textrm{RBF}$. These parameters differ under different domains and must be adjusted individually.
% We use a two-step strategy to find proper $\delta$ and $\alpha$. 
We find proper parameters by running two iterations without constraints and get two similar policies $\pi_0$ and $\pi_1$.
We record $r_\textrm{int}$ during training $\pi_1$ and the trend is shown in Fig.~\ref{fig:rint-trend}.
Not surprisingly, $r_\textrm{int}$ gradually decreases as training proceeds.

\textbf{Threshold} We set $\delta = c_1 D_\gS(\pi_0, \pi_1)$. We try several different $c_1 \in \{1, 1.2, 1.4, 1.6, 1.8, 2.0\}$ and find that $c_1 = 1.2$ or $1.4$ are universal proper solutions for all the experimental environments. 

\textbf{Intrinsic Scale Factor} We need to balance the intrinsic reward $r_{int}$ and the original reward $J$ so that neither of the two rewards can dominate the training process. Empirically, the maximums of the two rewards should be in the same order of magnitude. i.e., $\max_{\pi}J(\pi) = \alpha \times c_2\lambda_{max}\delta$, where $c_2 = O(1)$. When $c_2$ is too large, the new-trained policy $\pi_j$ will oscillate near the boundary of $D(\pi_i, \pi_j) = \delta$ for some pre-trained policy $p_i$. Conversely, when $c_2$ is too small, the intrinsic reward $r_{int}$ cannot yield diverse strategies. In experiments, we set $c_2 = 1.0$.

\textbf{Variance Factor} We sweep the variance factor across $\{1e-3,5e-3,1e-2,2e-2,1e-3\}$ by training $\pi_1$ and observe the trend of intrinsic rewards.
We find the steepest trend and select the corresponding $\sigma$. Empirically, we find that our algorithm performs robustly well when $\sigma^2=0.02$.

The $\delta$ and $\alpha$ of GRF and SMAC are listed in Table~\ref{tab:diversity hyperparameter}.

%%%%%%%%%%%%%%%%%%%%%%%%%%%%%%%%%%%%%%%%%%%%%%%%%%%%%%%%%%%%%%%%%%%%%%%%%%%
%%%%%%%%%%%%%%%%%%%%%%%%%%%%%%%%%%%%%%%%%%%%%%%%%%%%%%%%%%%%%%%%%%%%%%%%%%%
\subsection{Computation of Action-Based Measures in the Grid-World Example}

\begin{figure}[hb]
    \centering
    \includegraphics[width=0.85\textwidth]{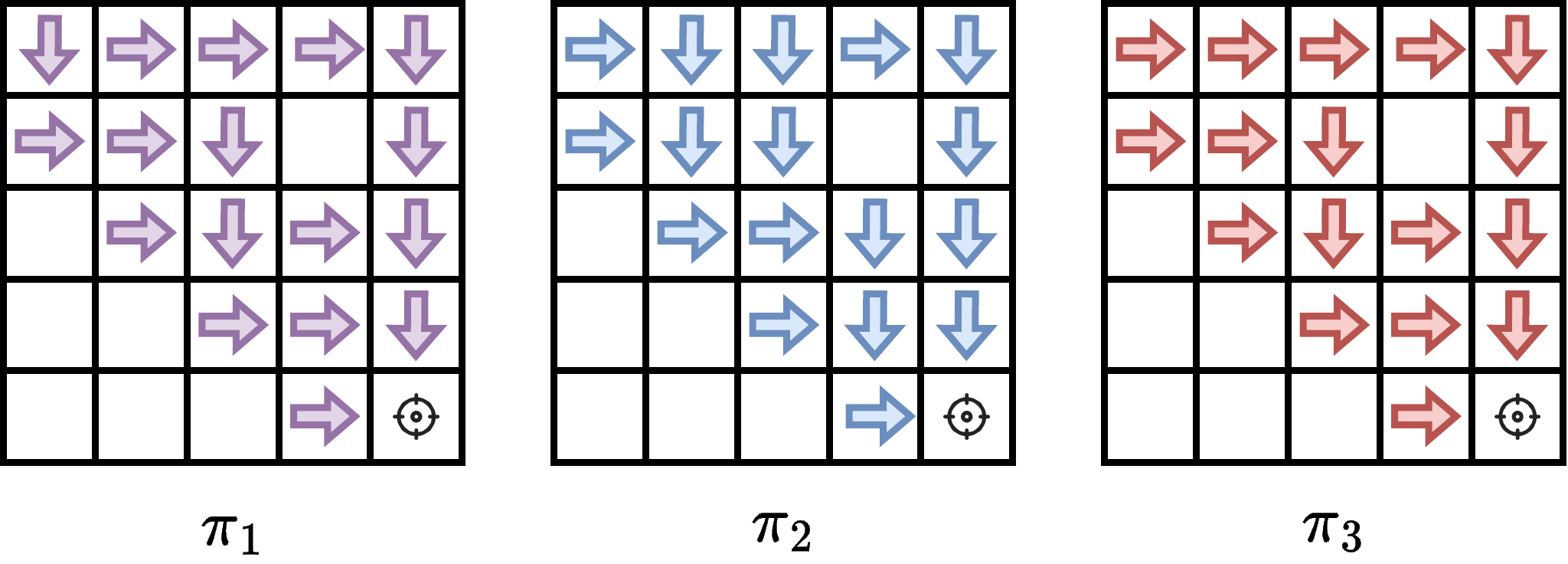}
    \caption{Policies in the grid-world example when $N_G=5$.}
    \label{fig:gw-policy}
\end{figure}

We consider the policies illustrated in Fig.~\ref{fig:gw-policy}.
These policies are all optimal since these actions only include ``right'' and ``down'' and actions on non-visited states can be arbitrary.
We only mark actions on states visited by any of these 3 policies and actions on other states can be considered the same.

\subsubsection{Action-Distribution-Based Measures}

Action-distribution-based diversity measures can be defined as
\begin{equation}
    D_\gA (\pi_i,\pi_j)=\E_{s\sim q(s)}\left[\tilde{D}\left(\pi_i(\cdot\mid s)\Vert\pi_j(\cdot\mid s)\right)\right],
\end{equation}
where $\tilde{D}(\cdot,\cdot):\triangle\times\triangle\to\R$ is a measure over action distributions and $q:\triangle(\gS)$ is a state distribution. Here, we consider $q$ to be the joint state distribution visited by $\pi_i$ and $\pi_j$.

\paragraph{KL Divergence} KL divergence is defined by
\begin{equation*}
    D_\textrm{KL}\left(\pi_i(\cdot\mid s),\pi_j(\cdot\mid s)\right)
    =\int_\gA \pi_i(a\mid s)\log\frac{\pi_i(a\mid s)}{\pi_j(a\mid s)}\textrm{d}a.
\end{equation*}
When $\pi_j(a\mid s)=0$ at any state $s$, KL divergence is $+\infty$.
Since the trajectories of these policies have disjoint states, $D_\gA^\textrm{KL}(\pi_1,\pi_2)=D_\gA^\textrm{KL}(\pi_1,\pi_3)=+\infty$.
Similar results can be obtained for cross-entropy.

\paragraph{$\textrm{JSD}_\gamma$}
$\textrm{JSD}_\gamma$ was defined in \cite{lupu_trajectory_2021} and we consider two special cases when $\gamma=0$ and $\gamma=1$.

As illustrated by \cite{lupu_trajectory_2021}, $\textrm{JSD}_0$ measures the expected number of times two policies will ``disagree'' by selecting different actions.
On trajectories induced by $\pi_1$ and $\pi_2$, there are $4+4$ states that $\pi_1$ disagrees with $\pi_2$ ($\pi_1$ and $\pi_2$ are symmetric) and $D_\gA^{\textrm{JSD}_0}(\pi_1,\pi_2)=8/16=1/2$.
Similarly, $\pi_1$ and $\pi_3$ only disagree at the initial state, therefore we have $D_\gA^{\textrm{JSD}_0}(\pi_1,\pi_3)=2/16=1/8$.

$\textrm{JSD}_1$ is defined by
\begin{align*}
    \textrm{JSD}_1(\pi_i,\pi_j)=&-\frac{1}{2}\sum_{\tau_i}P(\tau_i\mid\pi_i)\sum_{t=1}^T\frac{1}{T}\log\frac{\pi_i(\tau_i)+\pi_j(\tau_i)}{2\pi_i(\tau_i)}\\
    &-\frac{1}{2}\sum_{\tau_j}P(\tau_j\mid\pi_j)\sum_{t=1}^T\frac{1}{T}\log\frac{\pi_i(\tau_j)+\pi_j(\tau_j)}{2\pi_j(\tau_j)}.
\end{align*}
Since each of the policies considered only induces a single trajectory and $\pi_i(\tau_j)=0$ $(i\neq j)$, we can easily compute
\begin{align*}
    D_\gA^{\textrm{JSD}_1}(\pi_1,\pi_2)=D_\gA^{\textrm{JSD}_1}(\pi_1,\pi_3)=\log 2
\end{align*}

\paragraph{Wasserstein Distance}
Wasserstein distance or Earth Moving Distance (EMD) is 1 if two policies disagree on a state and 0 otherwise.
Therefore, it equals to $D_\gA^{\textrm{JSD}_0}$.

\subsubsection{Action Norm}

We embed the action ``right'' as vector $[1,0]$ since it increases the x-coordinate by 1 and the action ``down'' as vector $[0,-1]$ since it decreases the y-coordinate by 1.
This embedding can be naturally extended to a continuous action space with velocity actions.
Following ~\cite{parker-holder_effective_2020}, we compute the action norm over a uniform distribution on states.
We can see that there are $7$ states where $\pi_1$ and $\pi_2$ perform differently and $1$ state (the initial state) where $\pi_1$ and $\pi_3$ perform differently. Therefore, we can get $D(\pi_1,\pi_2)=\sqrt{7}$ and $D(\pi_1,\pi_3)=1$.

\subsubsection{State-Distance-Based Measures}

\paragraph{State $L_2$ Norm}
Similar to action $L_2$ norm, we concatenate the coordinates instead of actions as the embedding and compute the $L_2$ norm between embedding.

\paragraph{Wasserstein Distance}
Wasserstein distance is tractable in the grid-world example.
We consider 7 states (except the initial and final states) in each trajectory and compute the pair-wise distance as matrix $C$.
Then we solve the following linear programming
\begin{align*}
    \begin{split}
        \min_\gamma \quad &\sum_{i,j}\gamma\odot C\\
        \textrm{s.t.}\quad &\gamma\bm{1}=a,\,\gamma^T\bm{1}=b\\
        &\gamma_{i,j}\ge 0
    \end{split}
\end{align*}
where $\odot$ means element-wise multiplication, $\bm{1}$ is a all-one vector, $a=[\bm{1}^T,\bm{0}^T]^T$ and $b=[\bm{0}^T,\bm{1}^T]^T$ is the marginal state distribution of each policy.

%% file: A3-env.tex
\section{Environment Details}
\label{app:env}

\subsection{Details of the 2D Navigation Environment}

The navigation environment has an agent circle with size $a$ and 4 landmark circles with size $b$.
We pre-specify a threshold $c$ and constrain that the distance of final states reaching different landmarks must be larger than $c$.
Correspondingly, landmark circles are randomly initialized by constraining the pairwise distance between centers to be larger than a threshold $c+2(a+b)$ such that the final-state constraint is valid.
An episode ends if the agent touches any landmarks, i.e., the distance between the center of the agent and the center of the landmark $d<a+b$, or 1000 timesteps have elapsed.
The observation space includes the positions of the agent and all landmarks, which is a 10-dimensional vector.
The action space is a 2-dimensional vector, which is the agent velocity. The time interval is set to be $\Delta t=0.1$, i.e., the next position is computed by $x_{t+1}=x_t+\Delta t\cdot v$.
The reward is $1$ if the agent touches the landmark and $0$ otherwise.
% A policy controlling agent velocities is trained using PPO.

\subsection{Details of Environments}

We provide training configurations and environment introductions below and refer readers to our project website in App.~\ref{app:web} for visualizations of these environments.

\paragraph{Humanoid} We use the Humanoid environment in IsaacGym~\cite{makoviychuk2021isaac} with default observation and action spaces.
The input of intrinsic rewards or diversity measure is the observation without all torque states.

\paragraph{SMAC} We adopt the SMAC environment in the MAPPO codebase\footnote{\url{https://github.com/marlbenchmark/on-policy}} with the same configuration as~\citet{yu2021surprising}.
The input of intrinsic rewards or diversity measure is the state of all allies, including positions, health, etc.

On the ``easy'' map \emph{2m\_vs\_1z}, two marines must be controlled to defeat a Zealot. The marines can attack from a distance, while the Zealot's attacks are limited to close range. A successful strategy involves alternating the marines' attacks to distract the Zealot.
On the ``hard'' map \emph{2c\_vs\_64zg}, two colossi must be controlled by the agents to fight against 64 zergs. The colossi have a wider attack range and can move over cliffs. Strategies on this map may include hit-and-run tactics, waiting in corners, or dividing and conquering enemies.
The level of difficulty is determined by the learning performance of existing MARL algorithms. Harder maps require more exploration and training steps.

\paragraph{GRF} We adopt the ``simple115v2'' representation as observation with both ``scoring'' and ``checkpoint'' reward.
The reward is shared across all agents. The input of intrinsic rewards or diversity measure is the position and velocity of all attackers and the ball.
All policies are trained to control the left team to score against built-in bots.

\emph{academy\_3\_vs\_1\_with\_keeper}: In this scenario, a team of three players (left) tries to score a goal against a single defender and a goalkeeper. The left team starts with the ball and has to dribble past the defender and the goalkeeper to score a goal.

\emph{academy\_counterattack\_easy}: In this scenario, the left team starts with the ball in the front yard and tries to score a goal against several defenders. All eleven players in the left players can be controlled.

\emph{academy\_corner}: In this scenario, the left team tries to score a goal from a corner kick. The right team defends the goal and tries to prevent the left team from scoring. All eleven players in the left players can be controlled.

%% file: A4-imple.tex
\section{Implementation Details}
\label{app:imple}

\subsection{2D Navigation}

\begin{table}
    \centering
    \scriptsize
    \caption{Hyperparameters in the 2D navigation environment.}
    \label{tab:navi-hyper}
    \begin{tabular}{cccccccccccccccc}
    \toprule
     discount & GAE $\lambda$ & PPO epochs & clip parameter &  entropy bonus & $\lambda_\textrm{max}$ & actor lr & critic lr & Lagrange lr & batch size \\
     \midrule
     0.997 & 0.95 & 10 & 0.2 & 0 & 10 & 3e-4 & 1e-3 & 0.5 & 4000 \\
     \bottomrule
    \end{tabular}
\end{table}

We apply PPO with Lagrange multipliers to optimize the policy and hyperparameters are summarized in Table~\ref{tab:navi-hyper}.
$D(\pi_i,\pi_j)$ is simply taken as the $L_2$ distance of the final state reached by $\pi_i$ and $\pi_j$, i.e., $D(\pi_i,\pi_j)=\Vert s_H^{\pi_i}-s_H^{\pi_j}\Vert^2$.
The applied algorithm is the same as {\name} (see Appendix~\ref{app:algorithmic}) except that the intrinsic reward is only computed at the last timestep.
% We remark that this do not affect the overall algorithm because the extrinsic reward is sparse.

\subsection{{\name}}

In the $i$-th iteration ($1\le i\le M$), we learn an actor and a critic with $i$ separate value heads to accurately predict different return terms, including $i-1$ intrinsic returns for the diversity constraints and the environment reward.
We include all practical tricks mentioned in~\cite{yu2021surprising} because we find them all critical to algorithm performance.
% The detailed hyperparameters for SIPO are listed in \ref{tab:PPO-hyperparameters}.
We use separate actor and critic networks, both with hidden size 64 and a GRU layer with hidden size 64.
The common hyperparameters for {\name}, baselines, and ablations are listed in Table~\ref{tab:PPO-hyperparameters}.
Other environment-specific parameters, such as PPO epochs and mini-batch size, are all the same as~\cite{yu2021surprising}.
Besides, Table~\ref{tab:diversity hyperparameter} and Table~\ref{tab:SIPO-hyper} lists some extra hyperparameters for {\name}.

\begin{table}
\centering
\scriptsize
\begin{threeparttable}
\caption{Common hyperparameters for {\name}, baselines, and ablations.}
\label{tab:PPO-hyperparameters}
\begin{tabular}{ccccccccccccc}
\toprule  
discount & GAE $\lambda$ & actor lr & critic lr  & clip parameter & entropy bonus & GRF batch size & SMAC batch size \\
\midrule
0.99 & 0.95 & 5e-4 & 1e-3   & 0.2 & 0.01 & 9600 & 3200\\
\bottomrule
\end{tabular}
\end{threeparttable}
\end{table}

\begin{table}
\centering
\begin{threeparttable}
\caption{{\name} hyperparameters across all environments.}
\label{tab:SIPO-hyper}
\begin{tabular}{ccccccccc}
\toprule  
$\lambda_\textrm{max}$ & Discriminator lr & Lagrangian lr\\
\midrule
 $10$ & 4.0e-4 & $0.1$\\
\bottomrule
\end{tabular}
% \begin{tablenotes}\footnotesize
% \item[1] The percentage of the warm-up phase, where Lagrangian lr is always $0$.
% \end{tablenotes}
\end{threeparttable}
\end{table}

\subsection{Baselines}
\label{app:baseline}

We re-implement all baselines with PPO based on the MAPPO~\cite{yu2021surprising} project.
All algorithms run for the same number of environment frames.
Specific hyperparameters for baselines can be found in Appendix~\ref{app:baseline}.

\paragraph{SMERL}
SMERL trains a latent-conditioned policy that can robustly adapt to new scenarios. It promotes diversity by maximizing the mutual information between states and the latent variable. 
We implement SMERL with PPO, where the actor and the critic take as the input the concatenation of observation and a one-hot latent variable.
The discriminator is a 2-layer feed-forward network with 64 hidden units. The learning rate of the discriminator is the same as the learning rate of the critic network.
The input of the discriminator is the same as the input we use for {\namewd}.
The critic has 2 value heads for an accurate estimation of intrinsic return.
Since SMERL trains a single latent-conditioned policy, we train SMERL for $M\times$ more environment steps, such that total environment frames are the same.
The scaling factor of intrinsic rewards is $0.1$ and the threshold for diversification is $[0.81, 0.45, 0.72]$ ($0.9\times[0.9, 0.5, 0.8]$) for ``3v1'', ``counterattack'', and ``corner'' respectively.

\paragraph{DvD}
DvD simultaneously trains a population of policies to maximize the determinant of a kernel matrix based on action difference.
We concatenate the one-hot actions along a trajectory as the behavioral embedding.
The square of the variance factor, i.e., $\sigma^2$ in the RBF kernel, is set to be the length of behavioral embedding.
We also use the same Bayesian bandits as proposed in the original paper.
Training DvD in ``counterattack'' and ``corner'' exceeds the GPU memory and we exclude the results in the main body.

\paragraph{DIPG}
DIPG iteratively maximizes the maximum mean discrepancy (MMD) distance between the state distribution of the current policy and previously discovered policies.
For DIPG, we follow the open-source implementation\footnote{\url{https://github.com/dtak/DIPG-public}}.
We set the same variance factor in the RBF kernel as {\namerbf} and apply the same state as the input of the RBF kernel.
We sweep the coefficient of MMD loss among $\{0.1, 0.5, 0.9\}$ and find $0.1$ the most appropriate (larger value will cause training instability).
We use the same method to save archived trajectories as {\name} and the input of the RBF kernel is the same as the input we use for {\namerbf}.
To improve training efficiency, we only back-propagate the MMD loss at the first PPO epoch.

\paragraph{RSPO}
RSPO iteratively discovers diverse policies by optimizing extrinsic rewards on novel trajectories while optimizing diversity on other trajectories. The diversity measure is defined as the action-cross entropy along the trajectory.
For RSPO, we follow the opensource implementation\footnote{\url{https://github.com/footoredo/rspo-iclr-2022}} and use the same hyperparameters on the SMAC \emph{2c\_vs\_64zg} map in the original paper for GRF experiments.

\paragraph{TrajDi}
TrajDi was originally designed for cooperative multi-agent domains to facilitate zero-shot coordination. 
It defines a generalized Jensen-Shanon divergence objective between policy action distributions.
Then this objective and rewards are simultaneously optimized via population-based training.
We tried TrajDi in SMAC and GRF.
We sweep the action discount factor among $\{0.1, 0.5, 0.9\}$ and the coefficient of TrajDi loss among $\{0.1, 0.01, 0.001\}$.
However, TrajDi fails to converge in the ``3v1'' scenario and exceeds the GPU memory in the ``counterattack'' and ``corner'' scenarios.
Therefore, we exclude the performance of TrajDi in the main body.

\cmrd{
\paragraph{Domino}
We have meticulously re-implemented Domino within our codebase according to the appendix of \citet{domino}. We execute the algorithm in the Humanoid locomotion task, employing the robot state (excluding torques) for successor feature computation. Despite our earnest efforts to optimize Domino's performance, our findings reveal its comparable performance to SMERL, illustrated by a minimal diversity score of 0.01.
Therefore, we exclude the performance of Domino in the main body.
}

\cmrd{
\paragraph{APT}
APT maximized the nearest neighbor state-entropy estimation for skill discovery. While we also adopted this metric for diversity evaluation, there is a fundamental distinction in formulation. APT optimizes state entropy within a single policy, whereas our method, SIPO, targets the joint entropy of a population of policies. It is okay for each single policy within the population to have low state entropy.
To employ APT's objective of discovering diverse policies, training a population of agents concurrently is required. The algorithm should optimize the estimated entropy over states visited by all policies. Yet, this approach mandates large-scale k-NN computation (k=12) over substantial batches, leading to significant computational inefficiency. Despite our dedicated efforts, we didn’t finish a single training trial of APT within 48 hours (in contrast to other PBT baselines, e.g. DvD, which completes training in less than 8 hours).
}

\subsection{Ablation Study Details}
For the three ablation studies: fix-L, CE, and filter, we list the specific hyperparameters here:
\begin{itemize}
    \item fix-L: we set the Lagrange multiplier to be $0.2$; 
    \item CE: the threshold is $3.800$ and the intrinsic reward scale factor is $1/1000$ of that in the WD setting;
    \item filter: all the hyperparameters in the setting are the same as those in the WD setting.
\end{itemize}

%% file: A5-discussion.tex
\section{Discussion}
\label{app:discussion}

\subsection{The Failure Case of State-Distance-Based Diversity Measures}

A failure case of state-distance-based diversity measures may be when the state space includes many \emph{irrelevant features}. These features cannot reflect behavioral differences. If we run {\name} in such an environment, the learned strategies may be only diverse w.r.t these features and have little visual distinction. Like the famous noisy TV problem~\citep{noisytv}, the issue of irrelevant features is intrinsically challenging for general RL applications, which cannot be resolved by using action-based or state-occupancy-based diversity measures either.

Thanks to the advantages we discussed in the paper, we generally find that state-distance-based measures can be preferred in challenging RL problems. Meanwhile, since the state dimension can be much higher than actions, it is possible that RL optimization over states may be accordingly more difficult than actions. In practice, we can design a feature selector for those most relevant features for visual diversity and run diversity learning over the filtered features. In SMAC and GRF, we utilize the agent features (excluding enemies) as the input of diversity constraint without further modifications, as discussed in Appendix~\ref{app:imple}. We remark that even after filtering, the agent features remain high-dimensional while our algorithm still works well. Note that using a feature selector is a common practice in many existing domains, such as novelty search~\citep{cully_robots_2015}, exploration~\citep{DBLP:conf/icml/LiuJYS21}, and curriculum learning~\citep{amigo}. There are also works studying how to extract useful low-dimensional features from observations~\citep{laplacian,DBLP:conf/iclr/GhoshGL19}, which are orthogonal to our focus.

\subsection{The Distance Metric}

In Sec.~\ref{sec:method}, we adopt the two most popular implementations in the machine learning literature, i.e., RBF kernel and Wasserstein distance, while it is totally fine to adopt alternative implementations. For example, we can learn state representations (e.g. auto-encoder, Laplacian, or successor feature) and utilize pair-wise distance or norms as a diversity measure. Similar topics have been extensively discussed in the exploration literature~\citep{laplacian,exploration_succ}. We leave them as our future directions.

%% file: main.bbl
\begin{thebibliography}{74}
\providecommand{\natexlab}[1]{#1}
\providecommand{\url}[1]{\texttt{#1}}
\expandafter\ifx\csname urlstyle\endcsname\relax
  \providecommand{\doi}[1]{doi: #1}\else
  \providecommand{\doi}{doi: \begingroup \urlstyle{rm}\Url}\fi

\bibitem[Arjovsky et~al.(2017)Arjovsky, Chintala, and Bottou]{wgan}
Martin Arjovsky, Soumith Chintala, and L{\'e}on Bottou.
\newblock {W}asserstein generative adversarial networks.
\newblock In Doina Precup and Yee~Whye Teh, editors, \emph{Proceedings of the
  34th International Conference on Machine Learning}, volume~70 of
  \emph{Proceedings of Machine Learning Research}, pages 214--223. PMLR, 06--11
  Aug 2017.
\newblock URL \url{https://proceedings.mlr.press/v70/arjovsky17a.html}.

\bibitem[Babes et~al.(2008)Babes, de~Cote, and Littman]{babes_social_2008}
Monica Babes, Enrique~Munoz de~Cote, and Michael~L. Littman.
\newblock Social reward shaping in the prisoner's dilemma.
\newblock In Lin Padgham, David~C. Parkes, J{\"{o}}rg~P. M{\"{u}}ller, and
  Simon Parsons, editors, \emph{7th International Joint Conference on
  Autonomous Agents and Multiagent Systems {(AAMAS} 2008), Estoril, Portugal,
  May 12-16, 2008, Volume 3}, pages 1389--1392. {IFAAMAS}, 2008.
\newblock URL \url{https://dl.acm.org/citation.cfm?id=1402880}.

\bibitem[Baker et~al.(2020)Baker, Kanitscheider, Markov, Wu, Powell, McGrew,
  and Mordatch]{baker_emergent_2020}
Bowen Baker, Ingmar Kanitscheider, Todor~M. Markov, Yi~Wu, Glenn Powell, Bob
  McGrew, and Igor Mordatch.
\newblock Emergent tool use from multi-agent autocurricula.
\newblock In \emph{8th International Conference on Learning Representations,
  {ICLR} 2020, Addis Ababa, Ethiopia, April 26-30, 2020}. OpenReview.net, 2020.
\newblock URL \url{https://openreview.net/forum?id=SkxpxJBKwS}.

\bibitem[Bellemare et~al.(2016)Bellemare, Srinivasan, Ostrovski, Schaul,
  Saxton, and Munos]{bellemare2016countexploration}
Marc Bellemare, Sriram Srinivasan, Georg Ostrovski, Tom Schaul, David Saxton,
  and Remi Munos.
\newblock Unifying count-based exploration and intrinsic motivation.
\newblock \emph{Advances in neural information processing systems}, 29, 2016.

\bibitem[Burda and Edwards(2018)]{noisytv}
Yura Burda and Harri Edwards, Oct 2018.
\newblock URL
  \url{https://openai.com/blog/reinforcement-learning-with-prediction-based-rewards/}.

\bibitem[Burda et~al.(2018)Burda, Edwards, Storkey, and
  Klimov]{burda2018exploration}
Yuri Burda, Harrison Edwards, Amos Storkey, and Oleg Klimov.
\newblock Exploration by random network distillation.
\newblock \emph{arXiv preprint arXiv:1810.12894}, 2018.

\bibitem[Campero et~al.(2021)Campero, Raileanu, K{\"{u}}ttler, Tenenbaum,
  Rockt{\"{a}}schel, and Grefenstette]{amigo}
Andres Campero, Roberta Raileanu, Heinrich K{\"{u}}ttler, Joshua~B. Tenenbaum,
  Tim Rockt{\"{a}}schel, and Edward Grefenstette.
\newblock Learning with amigo: Adversarially motivated intrinsic goals.
\newblock In \emph{9th International Conference on Learning Representations,
  {ICLR} 2021, Virtual Event, Austria, May 3-7, 2021}. OpenReview.net, 2021.
\newblock URL \url{https://openreview.net/forum?id=ETBc\_MIMgoX}.

\bibitem[Campos et~al.(2020)Campos, Trott, Xiong, Socher, Gir{\'o}-i Nieto, and
  Torres]{campos2020explore}
V{\'\i}ctor Campos, Alexander Trott, Caiming Xiong, Richard Socher, Xavier
  Gir{\'o}-i Nieto, and Jordi Torres.
\newblock Explore, discover and learn: Unsupervised discovery of state-covering
  skills.
\newblock In \emph{International Conference on Machine Learning}, pages
  1317--1327. PMLR, 2020.

\bibitem[Charakorn et~al.(2022)Charakorn, Manoonpong, and
  Dilokthanakul]{charakorn2022incompatible}
Rujikorn Charakorn, Poramate Manoonpong, and Nat Dilokthanakul.
\newblock Generating diverse cooperative agents by learning incompatible
  policies.
\newblock In \emph{ICML 2022 Workshop AI for Agent-Based Modelling}, 2022.

\bibitem[Clark and Amodei(2016)]{boatracing}
Jack Clark and Dario Amodei, Dec 2016.
\newblock URL \url{https://openai.com/blog/faulty-reward-functions/}.

\bibitem[Cui et~al.(2023)Cui, Lupu, Sokota, Hu, Wu, and
  Foerster]{cui2023adversarialhanani}
Brandon Cui, Andrei Lupu, Samuel Sokota, Hengyuan Hu, David~J Wu, and
  Jakob~Nicolaus Foerster.
\newblock Adversarial diversity in hanabi.
\newblock In \emph{The Eleventh International Conference on Learning
  Representations}, 2023.

\bibitem[Cully et~al.(2015)Cully, Clune, Tarapore, and
  Mouret]{cully_robots_2015}
Antoine Cully, Jeff Clune, Danesh Tarapore, and Jean-Baptiste Mouret.
\newblock Robots that can adapt like animals.
\newblock \emph{Nature}, 521\penalty0 (7553):\penalty0 503--507, May 2015.
\newblock ISSN 0028-0836, 1476-4687.
\newblock \doi{10.1038/nature14422}.
\newblock URL \url{http://www.nature.com/articles/nature14422}.

\bibitem[Deb and Saha(2010)]{deb_finding_2010}
Kalyanmoy Deb and Amit Saha.
\newblock Finding multiple solutions for multimodal optimization problems using
  a multi-objective evolutionary approach.
\newblock In Martin Pelikan and J{\"{u}}rgen Branke, editors, \emph{Genetic and
  Evolutionary Computation Conference, {GECCO} 2010, Proceedings, Portland,
  Oregon, USA, July 7-11, 2010}, pages 447--454. {ACM}, 2010.
\newblock \doi{10.1145/1830483.1830568}.
\newblock URL \url{https://doi.org/10.1145/1830483.1830568}.

\bibitem[Deng et~al.(2022)Deng, Jang, and Ahn]{prototypical-repr}
Fei Deng, Ingook Jang, and Sungjin Ahn.
\newblock Dreamerpro: Reconstruction-free model-based reinforcement learning
  with prototypical representations.
\newblock In Kamalika Chaudhuri, Stefanie Jegelka, Le~Song, Csaba
  Szepesv{\'{a}}ri, Gang Niu, and Sivan Sabato, editors, \emph{International
  Conference on Machine Learning, {ICML} 2022, 17-23 July 2022, Baltimore,
  Maryland, {USA}}, volume 162 of \emph{Proceedings of Machine Learning
  Research}, pages 4956--4975. {PMLR}, 2022.
\newblock URL \url{https://proceedings.mlr.press/v162/deng22a.html}.

\bibitem[Devlin and Kudenko(2011)]{devlin_theoretical_2011}
Sam Devlin and Daniel Kudenko.
\newblock Theoretical considerations of potential-based reward shaping for
  multi-agent systems.
\newblock In Liz Sonenberg, Peter Stone, Kagan Tumer, and Pinar Yolum, editors,
  \emph{10th International Conference on Autonomous Agents and Multiagent
  Systems {(AAMAS} 2011), Taipei, Taiwan, May 2-6, 2011, Volume 1-3}, pages
  225--232. {IFAAMAS}, 2011.
\newblock URL
  \url{http://portal.acm.org/citation.cfm?id=2030503\&CFID=69153967\&CFTOKEN=38069692}.

\bibitem[Eysenbach et~al.(2019)Eysenbach, Gupta, Ibarz, and
  Levine]{eysenbach_diversity_2018}
Benjamin Eysenbach, Abhishek Gupta, Julian Ibarz, and Sergey Levine.
\newblock Diversity is all you need: Learning skills without a reward function.
\newblock In \emph{7th International Conference on Learning Representations,
  {ICLR} 2019, New Orleans, LA, USA, May 6-9, 2019}. OpenReview.net, 2019.
\newblock URL \url{https://openreview.net/forum?id=SJx63jRqFm}.

\bibitem[Fu et~al.(2022)Fu, Yu, Xu, Yang, and Wu]{fu2022revisiting}
Wei Fu, Chao Yu, Zelai Xu, Jiaqi Yang, and Yi~Wu.
\newblock Revisiting some common practices in cooperative multi-agent
  reinforcement learning.
\newblock In Kamalika Chaudhuri, Stefanie Jegelka, Le~Song, Csaba Szepesvari,
  Gang Niu, and Sivan Sabato, editors, \emph{Proceedings of the 39th
  International Conference on Machine Learning}, volume 162 of
  \emph{Proceedings of Machine Learning Research}, pages 6863--6877. PMLR,
  17--23 Jul 2022.
\newblock URL \url{https://proceedings.mlr.press/v162/fu22d.html}.

\bibitem[Ghosh et~al.(2019)Ghosh, Gupta, and Levine]{DBLP:conf/iclr/GhoshGL19}
Dibya Ghosh, Abhishek Gupta, and Sergey Levine.
\newblock Learning actionable representations with goal conditioned policies.
\newblock In \emph{7th International Conference on Learning Representations,
  {ICLR} 2019, New Orleans, LA, USA, May 6-9, 2019}. OpenReview.net, 2019.
\newblock URL \url{https://openreview.net/forum?id=Hye9lnCct7}.

\bibitem[Gupta et~al.(2021)Gupta, Savarese, Ganguli, and
  Fei-Fei]{gupta2021embodied}
Agrim Gupta, Silvio Savarese, Surya Ganguli, and Li~Fei-Fei.
\newblock Embodied intelligence via learning and evolution.
\newblock \emph{Nature communications}, 12\penalty0 (1):\penalty0 1--12, 2021.

\bibitem[Hazan et~al.(2019)Hazan, Kakade, Singh, and
  Van~Soest]{hazan2019maxentexpl}
Elad Hazan, Sham Kakade, Karan Singh, and Abby Van~Soest.
\newblock Provably efficient maximum entropy exploration.
\newblock In \emph{International Conference on Machine Learning}, pages
  2681--2691. PMLR, 2019.

\bibitem[Hu et~al.(2022)Hu, Xie, Liang, and Chang]{marlrolediag}
Siyi Hu, Chuanlong Xie, Xiaodan Liang, and Xiaojun Chang.
\newblock Policy diagnosis via measuring role diversity in cooperative
  multi-agent {RL}.
\newblock In Kamalika Chaudhuri, Stefanie Jegelka, Le~Song, Csaba
  Szepesv{\'{a}}ri, Gang Niu, and Sivan Sabato, editors, \emph{International
  Conference on Machine Learning, {ICML} 2022, 17-23 July 2022, Baltimore,
  Maryland, {USA}}, volume 162 of \emph{Proceedings of Machine Learning
  Research}, pages 9041--9071. {PMLR}, 2022.
\newblock URL \url{https://proceedings.mlr.press/v162/hu22c.html}.

\bibitem[Jaderberg et~al.(2017)Jaderberg, Dalibard, Osindero, Czarnecki,
  Donahue, Razavi, Vinyals, Green, Dunning, Simonyan, Fernando, and
  Kavukcuoglu]{jaderberg_population_2017}
Max Jaderberg, Valentin Dalibard, Simon Osindero, Wojciech~M. Czarnecki, Jeff
  Donahue, Ali Razavi, Oriol Vinyals, Tim Green, Iain Dunning, Karen Simonyan,
  Chrisantha Fernando, and Koray Kavukcuoglu.
\newblock Population {Based} {Training} of {Neural} {Networks}, November 2017.
\newblock URL \url{http://arxiv.org/abs/1711.09846}.
\newblock arXiv:1711.09846 [cs].

\bibitem[Jaderberg et~al.(2019)Jaderberg, Czarnecki, Dunning, Marris, Lever,
  Castañeda, Beattie, Rabinowitz, Morcos, Ruderman, Sonnerat, Green, Deason,
  Leibo, Silver, Hassabis, Kavukcuoglu, and
  Graepel]{jaderberg_human-level_2019}
Max Jaderberg, Wojciech~M. Czarnecki, Iain Dunning, Luke Marris, Guy Lever,
  Antonio~Garcia Castañeda, Charles Beattie, Neil~C. Rabinowitz, Ari~S.
  Morcos, Avraham Ruderman, Nicolas Sonnerat, Tim Green, Louise Deason, Joel~Z.
  Leibo, David Silver, Demis Hassabis, Koray Kavukcuoglu, and Thore Graepel.
\newblock Human-level performance in {3D} multiplayer games with
  population-based reinforcement learning.
\newblock \emph{Science}, 364\penalty0 (6443):\penalty0 859--865, May 2019.
\newblock ISSN 0036-8075, 1095-9203.
\newblock \doi{10.1126/science.aau6249}.
\newblock URL \url{https://www.science.org/doi/10.1126/science.aau6249}.

\bibitem[Jiang et~al.(2022)Jiang, Gao, and Chen]{jiang2022unsupervised}
Zheyuan Jiang, Jingyue Gao, and Jianyu Chen.
\newblock Unsupervised skill discovery via recurrent skill training.
\newblock \emph{Advances in Neural Information Processing Systems},
  35:\penalty0 39034--39046, 2022.

\bibitem[Kumar et~al.(2020)Kumar, Kumar, Levine, and Finn]{kumar_one_2020}
Saurabh Kumar, Aviral Kumar, Sergey Levine, and Chelsea Finn.
\newblock One solution is not all you need: Few-shot extrapolation via
  structured maxent {RL}.
\newblock In Hugo Larochelle, Marc'Aurelio Ranzato, Raia Hadsell,
  Maria{-}Florina Balcan, and Hsuan{-}Tien Lin, editors, \emph{Advances in
  Neural Information Processing Systems 33: Annual Conference on Neural
  Information Processing Systems 2020, NeurIPS 2020, December 6-12, 2020,
  virtual}, 2020.
\newblock URL
  \url{https://proceedings.neurips.cc/paper/2020/hash/5d151d1059a6281335a10732fc49620e-Abstract.html}.

\bibitem[Kurach et~al.(2020)Kurach, Raichuk, Stanczyk, Zajac, Bachem, Espeholt,
  Riquelme, Vincent, Michalski, Bousquet, and Gelly]{kurach2020google}
Karol Kurach, Anton Raichuk, Piotr Stanczyk, Michal Zajac, Olivier Bachem,
  Lasse Espeholt, Carlos Riquelme, Damien Vincent, Marcin Michalski, Olivier
  Bousquet, and Sylvain Gelly.
\newblock Google research football: {A} novel reinforcement learning
  environment.
\newblock In \emph{The Thirty-Fourth {AAAI} Conference on Artificial
  Intelligence, {AAAI} 2020, The Thirty-Second Innovative Applications of
  Artificial Intelligence Conference, {IAAI} 2020, The Tenth {AAAI} Symposium
  on Educational Advances in Artificial Intelligence, {EAAI} 2020, New York,
  NY, USA, February 7-12, 2020}, pages 4501--4510. {AAAI} Press, 2020.
\newblock URL \url{https://ojs.aaai.org/index.php/AAAI/article/view/5878}.

\bibitem[Lee et~al.(2019)Lee, Eysenbach, Parisotto, Xing, Levine, and
  Salakhutdinov]{state-marginal-matching}
Lisa Lee, Benjamin Eysenbach, Emilio Parisotto, Eric~P. Xing, Sergey Levine,
  and Ruslan Salakhutdinov.
\newblock Efficient exploration via state marginal matching.
\newblock \emph{CoRR}, abs/1906.05274, 2019.
\newblock URL \url{http://arxiv.org/abs/1906.05274}.

\bibitem[Lee et~al.(2022)Lee, Yao, and Finn]{lee_diversify_2022}
Yoonho Lee, Huaxiu Yao, and Chelsea Finn.
\newblock Diversify and disambiguate: Learning from underspecified data.
\newblock \emph{CoRR}, abs/2202.03418, 2022.
\newblock URL \url{https://arxiv.org/abs/2202.03418}.

\bibitem[Li et~al.(2021)Li, Wang, Wu, Zhao, Yang, and
  Zhang]{li_celebrating_2021}
Chenghao Li, Tonghan Wang, Chengjie Wu, Qianchuan Zhao, Jun Yang, and Chongjie
  Zhang.
\newblock Celebrating diversity in shared multi-agent reinforcement learning.
\newblock In Marc'Aurelio Ranzato, Alina Beygelzimer, Yann~N. Dauphin, Percy
  Liang, and Jennifer~Wortman Vaughan, editors, \emph{Advances in Neural
  Information Processing Systems 34: Annual Conference on Neural Information
  Processing Systems 2021, NeurIPS 2021, December 6-14, 2021, virtual}, pages
  3991--4002, 2021.
\newblock URL
  \url{https://proceedings.neurips.cc/paper/2021/hash/20aee3a5f4643755a79ee5f6a73050ac-Abstract.html}.

\bibitem[Li et~al.(2016)Li, Monroe, Ritter, Jurafsky, Galley, and
  Gao]{dialogue}
Jiwei Li, Will Monroe, Alan Ritter, Dan Jurafsky, Michel Galley, and Jianfeng
  Gao.
\newblock Deep reinforcement learning for dialogue generation.
\newblock In Jian Su, Xavier Carreras, and Kevin Duh, editors,
  \emph{Proceedings of the 2016 Conference on Empirical Methods in Natural
  Language Processing, {EMNLP} 2016, Austin, Texas, USA, November 1-4, 2016},
  pages 1192--1202. The Association for Computational Linguistics, 2016.
\newblock \doi{10.18653/v1/d16-1127}.
\newblock URL \url{https://doi.org/10.18653/v1/d16-1127}.

\bibitem[Lin et~al.(2020)Lin, Jin, and Jordan]{lin2020gradient}
Tianyi Lin, Chi Jin, and Michael~I. Jordan.
\newblock On gradient descent ascent for nonconvex-concave minimax problems.
\newblock In \emph{Proceedings of the 37th International Conference on Machine
  Learning, {ICML} 2020, 13-18 July 2020, Virtual Event}, volume 119 of
  \emph{Proceedings of Machine Learning Research}, pages 6083--6093. {PMLR},
  2020.
\newblock URL \url{http://proceedings.mlr.press/v119/lin20a.html}.

\bibitem[Liu and Abbeel(2021)]{liu2021behavior}
Hao Liu and Pieter Abbeel.
\newblock Behavior from the void: Unsupervised active pre-training.
\newblock \emph{Advances in Neural Information Processing Systems},
  34:\penalty0 18459--18473, 2021.

\bibitem[Liu et~al.(2021{\natexlab{a}})Liu, Jain, Yeh, and
  Schwing]{DBLP:conf/icml/LiuJYS21}
Iou{-}Jen Liu, Unnat Jain, Raymond~A. Yeh, and Alexander~G. Schwing.
\newblock Cooperative exploration for multi-agent deep reinforcement learning.
\newblock In Marina Meila and Tong Zhang, editors, \emph{Proceedings of the
  38th International Conference on Machine Learning, {ICML} 2021, 18-24 July
  2021, Virtual Event}, volume 139 of \emph{Proceedings of Machine Learning
  Research}, pages 6826--6836. {PMLR}, 2021{\natexlab{a}}.
\newblock URL \url{http://proceedings.mlr.press/v139/liu21j.html}.

\bibitem[Liu et~al.(2019)Liu, Lever, Merel, Tunyasuvunakool, Heess, and
  Graepel]{liu2019emergent}
Siqi Liu, Guy Lever, Josh Merel, Saran Tunyasuvunakool, Nicolas Heess, and
  Thore Graepel.
\newblock Emergent coordination through competition.
\newblock In \emph{7th International Conference on Learning Representations,
  {ICLR} 2019, New Orleans, LA, USA, May 6-9, 2019}. OpenReview.net, 2019.
\newblock URL \url{https://openreview.net/forum?id=BkG8sjR5Km}.

\bibitem[Liu et~al.(2022)Liu, Lever, Wang, Merel, Eslami, Hennes, Czarnecki,
  Tassa, Omidshafiei, Abdolmaleki, Siegel, Hasenclever, Marris,
  Tunyasuvunakool, Song, Wulfmeier, Muller, Haarnoja, Tracey, Tuyls, Graepel,
  and Heess]{liu2021motor}
Siqi Liu, Guy Lever, Zhe Wang, Josh Merel, S.~M.~Ali Eslami, Daniel Hennes,
  Wojciech~M. Czarnecki, Yuval Tassa, Shayegan Omidshafiei, Abbas Abdolmaleki,
  Noah~Y. Siegel, Leonard Hasenclever, Luke Marris, Saran Tunyasuvunakool,
  H.~Francis Song, Markus Wulfmeier, Paul Muller, Tuomas Haarnoja, Brendan~D.
  Tracey, Karl Tuyls, Thore Graepel, and Nicolas Heess.
\newblock From motor control to team play in simulated humanoid football.
\newblock \emph{Sci. Robotics}, 7\penalty0 (69), 2022.
\newblock \doi{10.1126/scirobotics.abo0235}.
\newblock URL \url{https://doi.org/10.1126/scirobotics.abo0235}.

\bibitem[Liu et~al.(2021{\natexlab{b}})Liu, Jia, Wen, Hu, Chen, Fan, Hu, and
  Yang]{liu_unifying_2021}
Xiangyu Liu, Hangtian Jia, Ying Wen, Yujing Hu, Yingfeng Chen, Changjie Fan,
  Zhipeng Hu, and Yaodong Yang.
\newblock Towards unifying behavioral and response diversity for open-ended
  learning in zero-sum games.
\newblock In Marc'Aurelio Ranzato, Alina Beygelzimer, Yann~N. Dauphin, Percy
  Liang, and Jennifer~Wortman Vaughan, editors, \emph{Advances in Neural
  Information Processing Systems 34: Annual Conference on Neural Information
  Processing Systems 2021, NeurIPS 2021, December 6-14, 2021, virtual}, pages
  941--952, 2021{\natexlab{b}}.
\newblock URL
  \url{https://proceedings.neurips.cc/paper/2021/hash/07bba581a2dd8d098a3be0f683560643-Abstract.html}.

\bibitem[Long et~al.(2020)Long, Zhou, Gupta, Fang, Wu, and
  Wang]{long2019evolutionary}
Qian Long, Zihan Zhou, Abhinav Gupta, Fei Fang, Yi~Wu, and Xiaolong Wang.
\newblock Evolutionary population curriculum for scaling multi-agent
  reinforcement learning.
\newblock In \emph{International Conference on Learning Representations}, 2020.

\bibitem[Lupu et~al.(2021)Lupu, Hu, and Foerster]{lupu_trajectory_2021}
Andrei Lupu, Hengyuan Hu, and Jakob~N. Foerster.
\newblock Trajectory diversity for zero-shot coordination.
\newblock In Frank Dignum, Alessio Lomuscio, Ulle Endriss, and Ann Now{\'{e}},
  editors, \emph{{AAMAS} '21: 20th International Conference on Autonomous
  Agents and Multiagent Systems, Virtual Event, United Kingdom, May 3-7, 2021},
  pages 1593--1595. {ACM}, 2021.
\newblock \doi{10.5555/3463952.3464170}.
\newblock URL
  \url{https://www.ifaamas.org/Proceedings/aamas2021/pdfs/p1593.pdf}.

\bibitem[Ma et~al.(2020)Ma, Du, and Matusik]{ma_efficient_2020}
Pingchuan Ma, Tao Du, and Wojciech Matusik.
\newblock Efficient continuous pareto exploration in multi-task learning.
\newblock In \emph{Proceedings of the 37th International Conference on Machine
  Learning, {ICML} 2020, 13-18 July 2020, Virtual Event}, volume 119 of
  \emph{Proceedings of Machine Learning Research}, pages 6522--6531. {PMLR},
  2020.
\newblock URL \url{http://proceedings.mlr.press/v119/ma20a.html}.

\bibitem[Ma(2021)]{ma_why_2021}
Tengyu Ma.
\newblock Why {Do} {Local} {Methods} {Solve} {Nonconvex} {Problems}?, March
  2021.
\newblock URL \url{http://arxiv.org/abs/2103.13462}.
\newblock arXiv:2103.13462 [cs, math, stat].

\bibitem[Machado et~al.(2020)Machado, Bellemare, and Bowling]{exploration_succ}
Marlos~C. Machado, Marc~G. Bellemare, and Michael Bowling.
\newblock Count-based exploration with the successor representation.
\newblock In \emph{The Thirty-Fourth {AAAI} Conference on Artificial
  Intelligence, {AAAI} 2020, The Thirty-Second Innovative Applications of
  Artificial Intelligence Conference, {IAAI} 2020, The Tenth {AAAI} Symposium
  on Educational Advances in Artificial Intelligence, {EAAI} 2020, New York,
  NY, USA, February 7-12, 2020}, pages 5125--5133. {AAAI} Press, 2020.
\newblock URL \url{https://ojs.aaai.org/index.php/AAAI/article/view/5955}.

\bibitem[Makoviychuk et~al.(2021)Makoviychuk, Wawrzyniak, Guo, Lu, Storey,
  Macklin, Hoeller, Rudin, Allshire, Handa, and State]{makoviychuk2021isaac}
Viktor Makoviychuk, Lukasz Wawrzyniak, Yunrong Guo, Michelle Lu, Kier Storey,
  Miles Macklin, David Hoeller, Nikita Rudin, Arthur Allshire, Ankur Handa, and
  Gavriel State.
\newblock Isaac gym: High performance gpu-based physics simulation for robot
  learning, 2021.

\bibitem[Masood and Doshi{-}Velez(2019)]{masood2019diversity}
Muhammad~A. Masood and Finale Doshi{-}Velez.
\newblock Diversity-inducing policy gradient: Using maximum mean discrepancy to
  find a set of diverse policies.
\newblock In Sarit Kraus, editor, \emph{Proceedings of the Twenty-Eighth
  International Joint Conference on Artificial Intelligence, {IJCAI} 2019,
  Macao, China, August 10-16, 2019}, pages 5923--5929. ijcai.org, 2019.
\newblock \doi{10.24963/ijcai.2019/821}.
\newblock URL \url{https://doi.org/10.24963/ijcai.2019/821}.

\bibitem[Miller and Shaw(1996)]{miller_genetic_1996}
B.L. Miller and M.J. Shaw.
\newblock Genetic algorithms with dynamic niche sharing for multimodal function
  optimization.
\newblock In \emph{Proceedings of {IEEE} {International} {Conference} on
  {Evolutionary} {Computation}}, pages 786--791, Nagoya, Japan, 1996. IEEE.
\newblock ISBN 978-0-7803-2902-7.
\newblock \doi{10.1109/ICEC.1996.542701}.
\newblock URL \url{http://ieeexplore.ieee.org/document/542701/}.

\bibitem[Mouret and Clune(2015)]{mouret_illuminating_2015}
Jean{-}Baptiste Mouret and Jeff Clune.
\newblock Illuminating search spaces by mapping elites.
\newblock \emph{CoRR}, abs/1504.04909, 2015.
\newblock URL \url{http://arxiv.org/abs/1504.04909}.

\bibitem[Ng et~al.(1999)Ng, Harada, and Russell]{ng1999policy}
Andrew~Y Ng, Daishi Harada, and Stuart Russell.
\newblock Policy invariance under reward transformations: Theory and
  application to reward shaping.
\newblock In \emph{Icml}, volume~99, pages 278--287, 1999.

\bibitem[Nilsson and Cully(2021)]{nilsson_policy_2021}
Olle Nilsson and Antoine Cully.
\newblock Policy gradient assisted {MAP}-{Elites}.
\newblock In \emph{Proceedings of the {Genetic} and {Evolutionary}
  {Computation} {Conference}}, pages 866--875, Lille France, June 2021. ACM.
\newblock ISBN 978-1-4503-8350-9.
\newblock \doi{10.1145/3449639.3459304}.
\newblock URL \url{https://dl.acm.org/doi/10.1145/3449639.3459304}.

\bibitem[Omidshafiei et~al.(2020)Omidshafiei, Tuyls, Czarnecki, Santos,
  Rowland, Connor, Hennes, Muller, P{\'e}rolat, Vylder,
  et~al.]{omidshafiei2020navigating}
Shayegan Omidshafiei, Karl Tuyls, Wojciech~M Czarnecki, Francisco~C Santos,
  Mark Rowland, Jerome Connor, Daniel Hennes, Paul Muller, Julien P{\'e}rolat,
  Bart~De Vylder, et~al.
\newblock Navigating the landscape of multiplayer games.
\newblock \emph{Nature communications}, 11\penalty0 (1):\penalty0 1--17, 2020.

\bibitem[Osa et~al.(2022)Osa, Tangkaratt, and Sugiyama]{osa_discovering_2022}
Takayuki Osa, Voot Tangkaratt, and Masashi Sugiyama.
\newblock Discovering diverse solutions in deep reinforcement learning by
  maximizing state-action-based mutual information.
\newblock \emph{Neural Networks}, 152:\penalty0 90--104, 2022.
\newblock \doi{10.1016/j.neunet.2022.04.009}.
\newblock URL \url{https://doi.org/10.1016/j.neunet.2022.04.009}.

\bibitem[Pacchiano et~al.(2020)Pacchiano, Parker-Holder, Tang, Choromanski,
  Choromanska, and Jordan]{pacchiano2020learning_to_score}
Aldo Pacchiano, Jack Parker-Holder, Yunhao Tang, Krzysztof Choromanski, Anna
  Choromanska, and Michael Jordan.
\newblock Learning to score behaviors for guided policy optimization.
\newblock In \emph{International Conference on Machine Learning}, pages
  7445--7454. PMLR, 2020.

\bibitem[Parker{-}Holder et~al.(2020{\natexlab{a}})Parker{-}Holder, Metz,
  Resnick, Hu, Lerer, Letcher, Peysakhovich, Pacchiano, and
  Foerster]{parker-holder_ridge_2020}
Jack Parker{-}Holder, Luke Metz, Cinjon Resnick, Hengyuan Hu, Adam Lerer,
  Alistair Letcher, Alexander Peysakhovich, Aldo Pacchiano, and Jakob~N.
  Foerster.
\newblock Ridge rider: Finding diverse solutions by following eigenvectors of
  the hessian.
\newblock In Hugo Larochelle, Marc'Aurelio Ranzato, Raia Hadsell,
  Maria{-}Florina Balcan, and Hsuan{-}Tien Lin, editors, \emph{Advances in
  Neural Information Processing Systems 33: Annual Conference on Neural
  Information Processing Systems 2020, NeurIPS 2020, December 6-12, 2020,
  virtual}, 2020{\natexlab{a}}.
\newblock URL
  \url{https://proceedings.neurips.cc/paper/2020/hash/08425b881bcde94a383cd258cea331be-Abstract.html}.

\bibitem[Parker{-}Holder et~al.(2020{\natexlab{b}})Parker{-}Holder, Pacchiano,
  Choromanski, and Roberts]{parker-holder_effective_2020}
Jack Parker{-}Holder, Aldo Pacchiano, Krzysztof~Marcin Choromanski, and
  Stephen~J. Roberts.
\newblock Effective diversity in population based reinforcement learning.
\newblock In Hugo Larochelle, Marc'Aurelio Ranzato, Raia Hadsell,
  Maria{-}Florina Balcan, and Hsuan{-}Tien Lin, editors, \emph{Advances in
  Neural Information Processing Systems 33: Annual Conference on Neural
  Information Processing Systems 2020, NeurIPS 2020, December 6-12, 2020,
  virtual}, 2020{\natexlab{b}}.
\newblock URL
  \url{https://proceedings.neurips.cc/paper/2020/hash/d1dc3a8270a6f9394f88847d7f0050cf-Abstract.html}.

\bibitem[Peysakhovich and Lerer(2018)]{peysakhovich_consequentialist_2018}
Alexander Peysakhovich and Adam Lerer.
\newblock Consequentialist conditional cooperation in social dilemmas with
  imperfect information.
\newblock In \emph{6th International Conference on Learning Representations,
  {ICLR} 2018, Vancouver, BC, Canada, April 30 - May 3, 2018, Conference Track
  Proceedings}. OpenReview.net, 2018.
\newblock URL \url{https://openreview.net/forum?id=BkabRiQpb}.

\bibitem[Pierrot et~al.(2022)Pierrot, Macé, Chalumeau, Flajolet, Cideron,
  Beguir, Cully, Sigaud, and Perrin-Gilbert]{pierrot_diversity_2022}
Thomas Pierrot, Valentin Macé, Félix Chalumeau, Arthur Flajolet, Geoffrey
  Cideron, Karim Beguir, Antoine Cully, Olivier Sigaud, and Nicolas
  Perrin-Gilbert.
\newblock Diversity {Policy} {Gradient} for {Sample} {Efficient}
  {Quality}-{Diversity} {Optimization}.
\newblock In \emph{Proceedings of the {Genetic} and {Evolutionary}
  {Computation} {Conference}}, pages 1075--1083, July 2022.
\newblock \doi{10.1145/3512290.3528845}.
\newblock URL \url{http://arxiv.org/abs/2006.08505}.
\newblock arXiv:2006.08505 [cs].

\bibitem[Pugh et~al.(2016)Pugh, Soros, and Stanley]{pugh_quality_2016}
Justin~K. Pugh, Lisa~B. Soros, and Kenneth~O. Stanley.
\newblock Quality {Diversity}: {A} {New} {Frontier} for {Evolutionary}
  {Computation}.
\newblock \emph{Frontiers in Robotics and AI}, 3, July 2016.
\newblock ISSN 2296-9144.
\newblock \doi{10.3389/frobt.2016.00040}.
\newblock URL
  \url{http://journal.frontiersin.org/Article/10.3389/frobt.2016.00040/abstract}.

\bibitem[Roughgarden(2020)]{roughgarden2020beyond}
Tim Roughgarden, editor.
\newblock \emph{Beyond the Worst-Case Analysis of Algorithms}.
\newblock Cambridge University Press, 2020.
\newblock ISBN 9781108637435.
\newblock \doi{10.1017/9781108637435}.
\newblock URL \url{https://doi.org/10.1017/9781108637435}.

\bibitem[Samvelyan et~al.(2019)Samvelyan, Rashid, de~Witt, Farquhar, Nardelli,
  Rudner, Hung, Torr, Foerster, and Whiteson]{samvelyan2019starcraft}
Mikayel Samvelyan, Tabish Rashid, Christian~Schr{\"{o}}der de~Witt, Gregory
  Farquhar, Nantas Nardelli, Tim G.~J. Rudner, Chia{-}Man Hung, Philip H.~S.
  Torr, Jakob~N. Foerster, and Shimon Whiteson.
\newblock The starcraft multi-agent challenge.
\newblock In Edith Elkind, Manuela Veloso, Noa Agmon, and Matthew~E. Taylor,
  editors, \emph{Proceedings of the 18th International Conference on Autonomous
  Agents and MultiAgent Systems, {AAMAS} '19, Montreal, QC, Canada, May 13-17,
  2019}, pages 2186--2188. International Foundation for Autonomous Agents and
  Multiagent Systems, 2019.
\newblock URL \url{http://dl.acm.org/citation.cfm?id=3332052}.

\bibitem[Singh et~al.(2003)Singh, Misra, Hnizdo, Fedorowicz, and
  Demchuk]{singh2003nearest}
Harshinder Singh, Neeraj Misra, Vladimir Hnizdo, Adam Fedorowicz, and Eugene
  Demchuk.
\newblock Nearest neighbor estimates of entropy.
\newblock \emph{American journal of mathematical and management sciences},
  23\penalty0 (3-4):\penalty0 301--321, 2003.

\bibitem[Spaan(2012)]{spaan2012partially}
Matthijs~TJ Spaan.
\newblock Partially observable markov decision processes.
\newblock In \emph{Reinforcement Learning}, pages 387--414. Springer, 2012.

\bibitem[Sun et~al.(2020)Sun, Peng, Dai, Guo, Lin, and Zhou]{sun2020novel}
Hao Sun, Zhenghao Peng, Bo~Dai, Jian Guo, Dahua Lin, and Bolei Zhou.
\newblock Novel policy seeking with constrained optimization.
\newblock \emph{arXiv preprint arXiv:2005.10696}, 2020.

\bibitem[Tang et~al.(2021)Tang, Yu, Chen, Xu, Wang, Fang, Du, Wang, and
  Wu]{tang_discovering_2021}
Zhenggang Tang, Chao Yu, Boyuan Chen, Huazhe Xu, Xiaolong Wang, Fei Fang,
  Simon~Shaolei Du, Yu~Wang, and Yi~Wu.
\newblock Discovering diverse multi-agent strategic behavior via reward
  randomization.
\newblock In \emph{9th International Conference on Learning Representations,
  {ICLR} 2021, Virtual Event, Austria, May 3-7, 2021}. OpenReview.net, 2021.
\newblock URL \url{https://openreview.net/forum?id=lvRTC669EY\_}.

\bibitem[Venturi et~al.(2018)Venturi, Bandeira, and Bruna]{venturi2018novalley}
Luca Venturi, Afonso~S. Bandeira, and Joan Bruna.
\newblock Neural networks with finite intrinsic dimension have no spurious
  valleys.
\newblock \emph{CoRR}, abs/1802.06384, 2018.
\newblock URL \url{http://arxiv.org/abs/1802.06384}.

\bibitem[Villani(2009)]{villani2009optimaltransport}
C{\'e}dric Villani.
\newblock \emph{Optimal transport: old and new}, volume 338.
\newblock Springer, 2009.

\bibitem[Vinyals et~al.(2019)Vinyals, Babuschkin, Czarnecki, Mathieu, Dudzik,
  Chung, Choi, Powell, Ewalds, Georgiev, Oh, Horgan, Kroiss, Danihelka, Huang,
  Sifre, Cai, Agapiou, Jaderberg, Vezhnevets, Leblond, Pohlen, Dalibard,
  Budden, Sulsky, Molloy, Paine, Gulcehre, Wang, Pfaff, Wu, Ring, Yogatama,
  Wünsch, McKinney, Smith, Schaul, Lillicrap, Kavukcuoglu, Hassabis, Apps, and
  Silver]{vinyals_grandmaster_2019}
Oriol Vinyals, Igor Babuschkin, Wojciech~M. Czarnecki, Michaël Mathieu, Andrew
  Dudzik, Junyoung Chung, David~H. Choi, Richard Powell, Timo Ewalds, Petko
  Georgiev, Junhyuk Oh, Dan Horgan, Manuel Kroiss, Ivo Danihelka, Aja Huang,
  Laurent Sifre, Trevor Cai, John~P. Agapiou, Max Jaderberg, Alexander~S.
  Vezhnevets, Rémi Leblond, Tobias Pohlen, Valentin Dalibard, David Budden,
  Yury Sulsky, James Molloy, Tom~L. Paine, Caglar Gulcehre, Ziyu Wang, Tobias
  Pfaff, Yuhuai Wu, Roman Ring, Dani Yogatama, Dario Wünsch, Katrina McKinney,
  Oliver Smith, Tom Schaul, Timothy Lillicrap, Koray Kavukcuoglu, Demis
  Hassabis, Chris Apps, and David Silver.
\newblock Grandmaster level in {StarCraft} {II} using multi-agent reinforcement
  learning.
\newblock \emph{Nature}, 575\penalty0 (7782):\penalty0 350--354, November 2019.
\newblock ISSN 0028-0836, 1476-4687.
\newblock \doi{10.1038/s41586-019-1724-z}.
\newblock URL \url{http://www.nature.com/articles/s41586-019-1724-z}.

\bibitem[Wang et~al.(2019)Wang, Lehman, Clune, and Stanley]{wang2019poet}
Rui Wang, Joel Lehman, Jeff Clune, and Kenneth~O Stanley.
\newblock Poet: open-ended coevolution of environments and their optimized
  solutions.
\newblock In \emph{Proceedings of the Genetic and Evolutionary Computation
  Conference}, pages 142--151, 2019.

\bibitem[Wang et~al.(2021)Wang, Gupta, Mahajan, Peng, Whiteson, and
  Zhang]{wang_rode_2020}
Tonghan Wang, Tarun Gupta, Anuj Mahajan, Bei Peng, Shimon Whiteson, and
  Chongjie Zhang.
\newblock {RODE:} learning roles to decompose multi-agent tasks.
\newblock In \emph{9th International Conference on Learning Representations,
  {ICLR} 2021, Virtual Event, Austria, May 3-7, 2021}. OpenReview.net, 2021.
\newblock URL \url{https://openreview.net/forum?id=TTUVg6vkNjK}.

\bibitem[Wu et~al.(2023)Wu, Yao, Fu, Tian, Qian, Yang, FU, and Wei]{wuquality}
Shuang Wu, Jian Yao, Haobo Fu, Ye~Tian, Chao Qian, Yaodong Yang, QIANG FU, and
  Yang Wei.
\newblock Quality-similar diversity via population based reinforcement
  learning.
\newblock In \emph{The Eleventh International Conference on Learning
  Representations}, 2023.

\bibitem[Wu et~al.(2019)Wu, Tucker, and Nachum]{laplacian}
Yifan Wu, George Tucker, and Ofir Nachum.
\newblock The laplacian in {RL:} learning representations with efficient
  approximations.
\newblock In \emph{7th International Conference on Learning Representations,
  {ICLR} 2019, New Orleans, LA, USA, May 6-9, 2019}. OpenReview.net, 2019.
\newblock URL \url{https://openreview.net/forum?id=HJlNpoA5YQ}.

\bibitem[Yu et~al.(2021)Yu, Velu, Vinitsky, Wang, Bayen, and
  Wu]{yu2021surprising}
Chao Yu, Akash Velu, Eugene Vinitsky, Yu~Wang, Alexandre Bayen, and Yi~Wu.
\newblock The surprising effectiveness of ppo in cooperative, multi-agent
  games.
\newblock \emph{arXiv preprint arXiv:2103.01955}, 2021.

\bibitem[Zahavy et~al.(2021)Zahavy, O'Donoghue, Barreto, Mnih, Flennerhag, and
  Singh]{zahavy2021discovering}
Tom Zahavy, Brendan O'Donoghue, Andre Barreto, Volodymyr Mnih, Sebastian
  Flennerhag, and Satinder Singh.
\newblock Discovering diverse nearly optimal policies with successor features.
\newblock \emph{arXiv preprint arXiv:2106.00669}, 2021.

\bibitem[Zahavy et~al.(2022)Zahavy, Schroecker, Behbahani, Baumli, Flennerhag,
  Hou, and Singh]{domino}
Tom Zahavy, Yannick Schroecker, Feryal M.~P. Behbahani, Kate Baumli, Sebastian
  Flennerhag, Shaobo Hou, and Satinder Singh.
\newblock Discovering policies with domino: Diversity optimization maintaining
  near optimality.
\newblock \emph{CoRR}, abs/2205.13521, 2022.
\newblock \doi{10.48550/arXiv.2205.13521}.
\newblock URL \url{https://doi.org/10.48550/arXiv.2205.13521}.

\bibitem[Zhang et~al.(2019)Zhang, Yu, and Turk]{zhang2019novel_task}
Yunbo Zhang, Wenhao Yu, and Greg Turk.
\newblock Learning novel policies for tasks.
\newblock In \emph{International Conference on Machine Learning}, pages
  7483--7492. PMLR, 2019.

\bibitem[Zhao et~al.(2021)Zhao, Song, Haifeng, Gao, Wu, Sun, and
  Wei]{zhao2021maximum}
Rui Zhao, Jinming Song, Hu~Haifeng, Yang Gao, Yi~Wu, Zhongqian Sun, and Yang
  Wei.
\newblock Maximum entropy population based training for zero-shot human-ai
  coordination.
\newblock \emph{arXiv preprint arXiv:2112.11701}, 2021.

\bibitem[Zhou et~al.(2022)Zhou, Fu, Zhang, and Wu]{zhou_continuously_2022}
Zihan Zhou, Wei Fu, Bingliang Zhang, and Yi~Wu.
\newblock Continuously discovering novel strategies via reward-switching policy
  optimization.
\newblock In \emph{The Tenth International Conference on Learning
  Representations, {ICLR} 2022, Virtual Event, April 25-29, 2022}.
  OpenReview.net, 2022.
\newblock URL \url{https://openreview.net/forum?id=hcQHRHKfN\_}.

\end{thebibliography}
